\documentclass{article} 

\usepackage{maniedit_preprint,times}

\usepackage{hyperref}
\usepackage{url}

\usepackage[T1]{fontenc}
\usepackage[utf8]{inputenc}

\usepackage{microtype}
\usepackage{inconsolata}

\usepackage{graphicx}

\usepackage{amsmath}
\usepackage{amsthm}
\usepackage{amssymb}
\usepackage{adjustbox}
\newtheorem{proposition}{Proposition}

\newtheorem{corollary}{Corollary}
\usepackage{booktabs}
\usepackage{multirow}
\usepackage[normalem]{ulem}
\usepackage{enumitem}

\renewcommand{\cite}{\citep}

\title{ManiEdit: Sequential Unstructured Knowledge Editing for Language Models from a Manifold Perspective}

\author{Rui Liu$^{1}$ \quad
Chenheng Zhang$^{1}$ \quad
Haoxuan Li$^{2}$\thanks{Corresponding authors: Zhouchen Lin (\texttt{zlin@pku.edu.cn}) and Haoxuan Li (\texttt{hxli@pku.edu.cn}).} \quad
Zhouchen Lin$^{1,2}$\footnotemark[2] \\
\normalfont $^{1}$State Key Lab of General Artificial Intelligence,\\
School of Intelligence Science and Technology, Peking University\\
$^{2}$Institute for Artificial Intelligence, Peking University
}

\begin{document}

\setcounter{footnote}{1}
\maketitle

\vspace{-5pt}
\begin{abstract}
    Large language models (LLMs) inevitably generate some incorrect or outdated content, necessitating efficient and precise mechanisms for continual knowledge updates. However, existing model editing methods struggle to sequentially edit unstructured long-form knowledge, suffering from severe edit forgetting and degradation of general capabilities. To address these challenges, we reframe knowledge editing from a manifold perspective, viewing it as a localized displacement of an edit sub-manifold within the global knowledge manifold. Under this formulation, the problem can be decomposed into two key questions: (i) how to identify representative edit points that effectively anchor the edit sub-manifold, and (ii) how to preserve the remaining manifold structure during the sub-manifold displacement process. Based on this perspective, we propose \textbf{ManiEdit}, a novel manifold-aware autoregressive editing framework consisting of two core components. \textit{Pivot Localization} addresses the mediocre-point dilemma by identifying high-leverage pivots to anchor the edit sub-manifold. \textit{Manifold-Aware Preservation} preserves different knowledge types through an energy-weighted penalty combined with recursive null-space alignment. Experiments on two base LLMs and four unstructured editing benchmarks demonstrate that ManiEdit achieves state-of-the-art performance, outperforming the strongest baseline by up to $+27.81$ BERTScore and $+8.50$ ROUGE-L, while maintaining near-original general capabilities across six representative downstream tasks.
    Our code is available at: https://github.com/Areyliu/ManiEdit
\end{abstract}

\vspace{-10pt}
\section{Introduction}
\label{sec:intro}
\vspace{-5pt}

Large Language Models (LLMs) have demonstrated impressive capabilities in storing and recalling vast amounts of knowledge within their parameters~\cite{brown2020gpt3,petroni2019lm,roberts2020knowledge,zhao2023survey}. However, stored facts often become inaccurate or outdated due to noisy training data and temporal drift~\cite{de-cao2021editing,mitchell2022fast}. While fine-tuning on updated facts offers a straightforward remedy, it is resource-intensive and risks both overfitting to individual examples and catastrophic forgetting~\cite{mitchell2022memory,meng2023locating}. To overcome these limitations, \emph{knowledge editing} has emerged as a promising alternative that enables precise, cost-effective modification of specific target knowledge while preserving other knowledge.

Current knowledge editing approaches roughly fall into two categories: (i) \emph{parameter-modifying} methods that directly adjust a small subset of model parameters~\cite{meng2023mass,fang2025alphaedit,jiang2025anyedit}, and (ii) \emph{parameter-preserving} methods that introduce additional external modules without updating the original weights~\cite{mitchell2022memory,hartvigsen2023aging,yu2023melo}. Among the former, the dominant locate-then-edit paradigm identifies the weights storing target knowledge via causal tracing and perturbs them with a closed-form update. Representative methods such as ROME~\cite{meng2023locating} and MEMIT~\cite{meng2023mass} interpret the down-projection matrix of each feed-forward layer as a linear associative memory and derive closed-form updates for atomic subject-predicate-object triplets.

Despite their success on structured triplets, realistic deployment still faces two fundamental challenges. First, roughly 80\% of real-world knowledge is \emph{unstructured}~\cite{bavota2016mining,deng2025everything}, expressed in diverse formats (e.g., poetry, code, mathematical derivation) with rich context dependencies rather than discrete triplets, making it difficult to cleanly locate, encode, and inject target knowledge into the model. Second, knowledge updates accumulate \emph{sequentially} over time rather than arriving in a single batch, requiring continual editing without disrupting prior updates or pre-trained capabilities. This sequential setting gives rise to two critical failure modes: \textit{edit forgetting}, where each new edit partially erases previously updated knowledge, causing editing performance to decay as the sequence grows; and \textit{general-capability collapse}, where accumulated parameter drift destroys the model's linguistic and reasoning abilities within only a few hundred edits (Figure~\ref{fig:motivation}).

To solve these challenges, we reframe knowledge editing from a manifold perspective, viewing it as a localized displacement of an edit sub-manifold within the global knowledge manifold. Under this formulation, the problem can be decomposed into two key questions: (i) how to identify representative edit points that effectively anchor the edit sub-manifold, and (ii) how to preserve the remaining manifold structure during the sub-manifold displacement process, thereby achieving precise editing while preserving historical edits and general capabilities.

From this perspective, we propose \textbf{ManiEdit}, a novel manifold-aware autoregressive editing paradigm designed for sequential unstructured knowledge editing, which comprises two tightly coupled components that address the key questions above. (i) \textbf{Pivot Localization} identifies \emph{semantic pivots}---tokens with high leverage scores on the manifold periphery---and then anchors the edit sub-manifold at these pivots, improving the effectiveness of each weight update. (ii) \textbf{Manifold-Aware Preservation} decouples the preservation of two types of knowledge according to their distinct manifold characteristics: an energy-weighted penalty preserves the geometry of pre-trained representations, while a recursive null-space alignment maintains safe orthogonality to historical edits, enabling stable retention of both sequentially edited knowledge and pre-trained general capabilities. The resulting unified objective admits a unique closed-form solution. We further establish three theoretical propositions that demonstrate the advantages of each component and characterize their synergy.

To validate the effectiveness of our method, we conduct extensive experiments using two base LLMs on four unstructured benchmarks. The results show that ManiEdit achieves state-of-the-art performance, surpassing the strongest competitor by up to $+27.81$ BERTScore and $+8.50$ ROUGE-L. After 500 sequential edits, ManiEdit retains general capabilities within $1$--$2\%$ of the pre-edited model on most of the representative downstream tasks, whereas all baselines collapse within $100$--$200$ edits. Further analysis experiments, ablation studies, and long-horizon stress tests confirm the robustness and efficacy of ManiEdit.

In summary, our main contributions are:
\begin{itemize}[leftmargin=*,nosep]
    \item We are the first to reframe knowledge editing as a localized displacement of an edit sub-manifold within the global manifold, decomposing it into the dual problems of edit point identification and manifold structure preservation.
    \item We propose \textit{Pivot Localization}, which identifies high-leverage semantic pivots that anchor the edit sub-manifold, and \textit{Manifold-Aware Preservation}, which decouples the preservation of two types of knowledge by their distinct manifold properties.
    \item Extensive experiments show ManiEdit achieves state-of-the-art performance on four unstructured editing benchmarks and retains near-original general capability on six downstream tasks.
\end{itemize}

\begin{figure}[t] 
  \centering
  \includegraphics[width=0.75\linewidth]{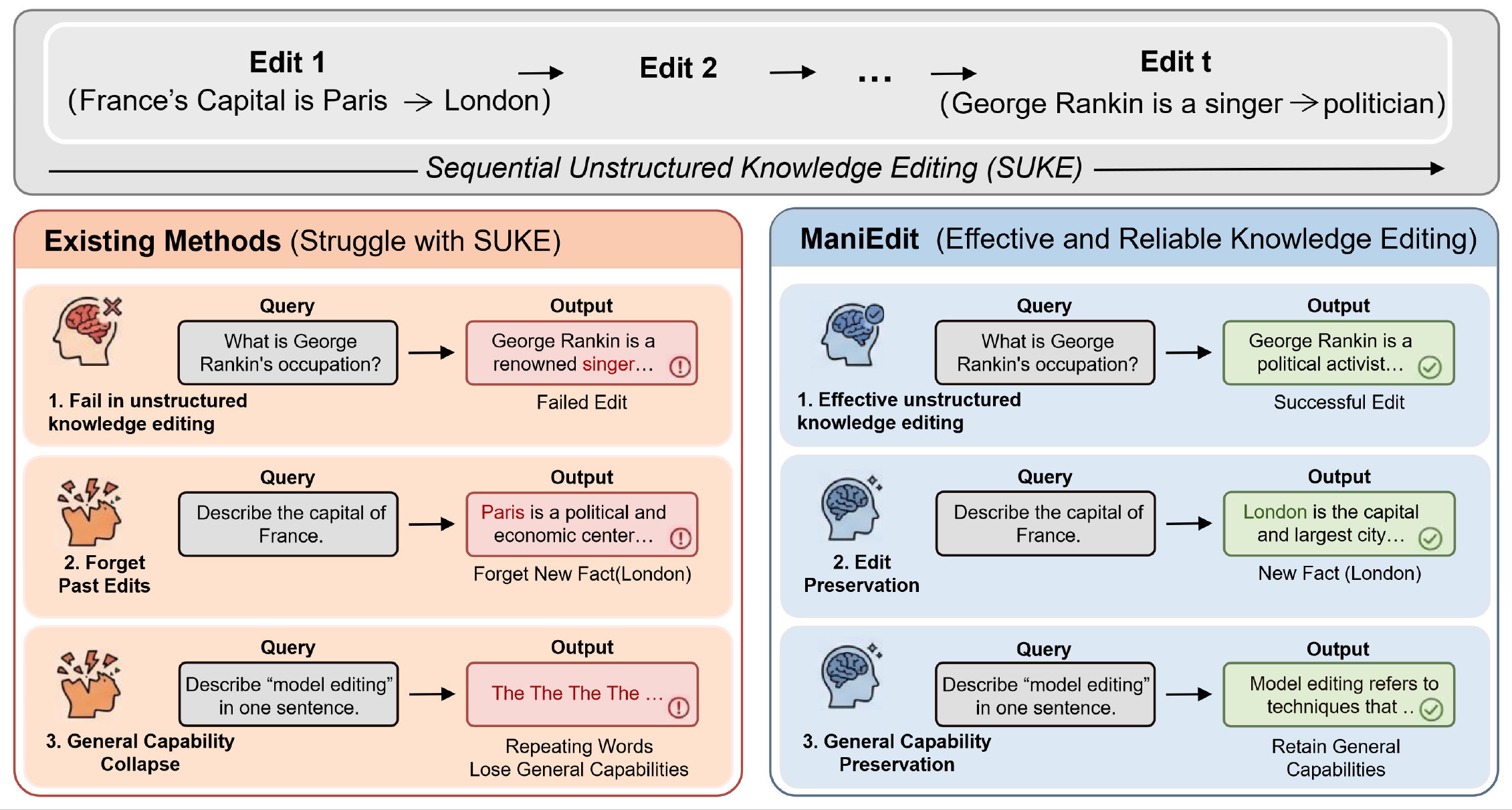}
  \caption{Motivation for ManiEdit. Existing methods struggle with sequential unstructured knowledge editing, suffering from edit forgetting and general-capability collapse (left). In contrast, ManiEdit enables effective and reliable knowledge editing, preserving both edit knowledge and pre-trained general capabilities (right).}
  \label{fig:motivation}
  \vspace{-10pt}
\end{figure}

\section{Preliminary}
\label{sec:preliminary}

An autoregressive large language model (LLM) predicts the next token $x$ in a sequence based on all preceding tokens. At layer $l$, the hidden state $h^l$ of token $x$ is computed as:
\begin{equation}
h^l = h^{l-1} + a^l + m^l,
\end{equation}
where $a^l$ and $m^l$ denote the outputs of the multi-head self-attention (MHSA) and feed-forward network (FFN) modules, respectively. The FFN output is computed via a two-layer projection with a non-linear activation $\sigma$ and layer normalization $\gamma$:
\begin{equation}
\label{eq:ffn}
\underbrace{m^l}_{v} = W^l_{\text{down}} \; \underbrace{\sigma \!\left( W^l_{\text{up}} \; \gamma(h^{l-1} + a^l) \right)}_{k},
\end{equation}
where $W^l_{\text{up}} \in \mathbb{R}^{d_{\text{ff}} \times d}$ is the up-projection matrix, $W^l_{\text{down}} \in \mathbb{R}^{d \times d_{\text{ff}}}$ is the down-projection matrix, $k \in \mathbb{R}^{d_{\text{ff}}}$ is the activation serving as the \textit{key}, and $v \in \mathbb{R}^{d}$ is the FFN output serving as the \textit{value}.

Following \citet{geva2021transformer} and \citet{meng2023locating}, $W_{\text{down}}$ can be interpreted as a \textit{linear associative memory} that functions as key-value storage for factual knowledge retrieval. For instance, the key may encode the contextual representation of a query (e.g., ``The capital of France is''), while the corresponding value encodes the associated answer (e.g., ``Paris''). This perspective has inspired the \textit{locate-then-edit} paradigm for knowledge editing, where targeted modifications to $W_{\text{down}}$ enable precise updates to specific stored associations. Following \citet{deng2025everything} and \citet{jiang2025anyedit}, we use the intermediate activation at the last token of each input chunk as the key $k$ for weight updates throughout this work. The locate-then-edit procedure and its extension to long-form knowledge are detailed in Appendix~\ref{app:prelim}. For notational simplicity, we use $W$ to denote $W_{\text{down}}$ in subsequent sections. 

\section{Methodology}
\label{sec:method}

We now present \textbf{ManiEdit}, a paradigm for sequential unstructured knowledge editing. We first reframe editing from a representation-manifold perspective, then introduce two core components: \emph{Pivot Localization} for identifying informative semantic anchors in unstructured knowledge and \emph{Manifold-Aware Preservation} for geometry-aware knowledge protection (Figure~\ref{fig:ManiEdit_overview}). We finally derive a unified objective and analyze the theoretical advantages.

\subsection{A Manifold Perspective of Model Editing}

A well-established line of research has revealed that the hidden representations of LLMs exhibit rich geometric structure: activations concentrate along a structured, low-dimensional manifold~\cite{mamou2020emergence,modell2025origins,mabrok2026latent}, whose principal axes encode the statistical regularities of language. Disrupting this learned representation geometry can therefore impair the model's general capabilities~\cite{nishi2025representation_shattering}. We build our geometric formulation of knowledge editing on this empirical foundation.

\paragraph{Input and output representation manifolds.}
Following the manifold hypothesis~\cite{bengio2013representation,fefferman2016testing}, pre-trained key activations $k$ do not fill the ambient key space uniformly but concentrate near a low-dimensional \emph{input representation manifold} $\mathcal{M}_{\text{in}}$, a geometric structure whose principal directions are shaped by pre-training. We characterize this geometry via the second-moment matrix $\Sigma = \mathbb{E}[kk^\top]$, whose eigendecomposition $\Sigma = \Psi\Lambda\Psi^\top$ provides a linear approximation to the manifold's structure: an orthonormal basis $\Psi = [\psi_1,\dots,\psi_n]$ with energy spectrum $\Lambda = \operatorname{diag}(\lambda_1,\dots,\lambda_n)$. Our framework retains the full spectrum and lets each eigenvalue $\lambda_j$ encode the relative geometric importance of the corresponding direction. The weight $W$ then maps $\mathcal{M}_{\text{in}}$ to the \emph{output representation manifold}
\begin{equation*}
\mathcal{M}_{\text{out}} \;=\; W(\mathcal{M}_{\text{in}}),
\end{equation*}
which encodes the associated factual knowledge as a geometric object in the output space. This two-manifold perspective is illustrated in Figure~\ref{fig:spectral_perspective}.

\begin{figure}[t]
  \centering
  \includegraphics[width=0.75\linewidth]{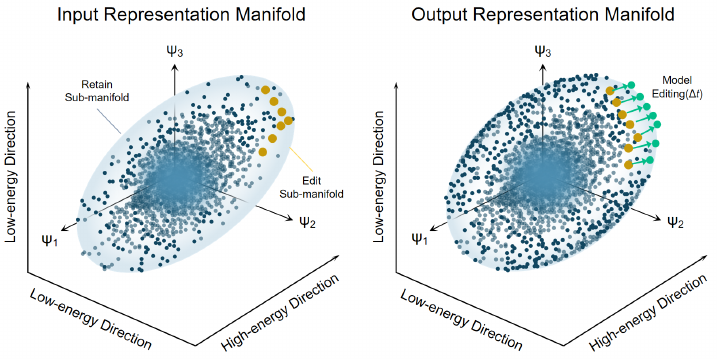}
  \caption{A manifold perspective of model editing. \textbf{Left}: The input representation manifold $\mathcal{M}_{\text{in}}$. Retain keys (blue) distribute along the principal axes $\{\psi_j\}$, forming a global sub-manifold $\mathcal{M}^{\text{retain}}_{\text{in}}$; edit-related keys (gold) form a local sub-manifold $\mathcal{M}^{\text{edit}}_{\text{in}}$. \textbf{Right}: The output representation manifold $\mathcal{M}_{\text{out}} = W(\mathcal{M}_{\text{in}})$. The edit operator $\mathcal{E}$ displaces the edit sub-manifold $\mathcal{M}^{\text{edit}}_{\text{out}}$ (gold) to its target $\mathcal{M}^{*,\text{edit}}_{\text{out}}$ (green), while the retain sub-manifold $\mathcal{M}^{\text{retain}}_{\text{out}}$ (blue) remains unchanged.}
  \label{fig:spectral_perspective}
  \vspace{-10pt}
\end{figure}

\paragraph{Knowledge editing as localized sub-manifold displacement.}
Let $\mathcal{M}^{\text{edit}}_{\text{in}} \subset \mathcal{M}_{\text{in}}$ be the \emph{edit sub-manifold}---the region of $\mathcal{M}_{\text{in}}$ associated with the target knowledge to be modified---and let $\mathcal{M}^{\text{retain}}_{\text{in}} = \mathcal{M}_{\text{in}} \setminus \mathcal{M}^{\text{edit}}_{\text{in}}$ be the complementary \emph{retain sub-manifold}. Their images under $W$ partition the output manifold: $\mathcal{M}^{\text{edit}}_{\text{out}} = W(\mathcal{M}^{\text{edit}}_{\text{in}})$ and $\mathcal{M}^{\text{retain}}_{\text{out}} = W(\mathcal{M}^{\text{retain}}_{\text{in}})$. We define the model editing operator $\mathcal{E}: W' = \mathcal{E}(W)$ as a \emph{localized sub-manifold displacement}, requiring $W'$ to satisfy:
\begin{equation}
\label{eq:edit_goal}
W'\!\bigl(\mathcal{M}^{\text{edit}}_{\text{in}}\bigr) = \mathcal{M}^{*,\text{edit}}_{\text{out}},
\qquad
W'\!\bigl(\mathcal{M}^{\text{retain}}_{\text{in}}\bigr) = \mathcal{M}^{\text{retain}}_{\text{out}}.
\end{equation}
The first condition displaces $\mathcal{M}^{\text{edit}}_{\text{out}}$ (gold) to the target $\mathcal{M}^{*,\text{edit}}_{\text{out}}$ (green), while the second holds $\mathcal{M}^{\text{retain}}_{\text{out}}$ (blue) invariant. 
The fundamental challenge of sequential long-form editing is therefore threefold:
(i) to locate representative keys that anchor the edit sub-manifold $\mathcal{M}^{\text{edit}}_{\text{in}}$ to facilitate precise editing;
(ii) to realize such localized displacements faithfully without distorting the global structure of $\mathcal{M}_{\text{out}}$;
and (iii) to prevent subsequent edits from perturbing established displacements.

\subsection{Pivot Localization}
\label{sec:leverage}

Achieving the edit objective in Equation~\ref{eq:edit_goal} first requires addressing a challenge specific to unstructured long-form knowledge editing: how to select the representative keys that anchor each sample's edit sub-manifold $\mathcal{M}^{\text{edit}}_{\text{in}}$.

\paragraph{The Mediocre-Point Dilemma.}
To handle long-form knowledge, AnyEdit~\cite{jiang2025anyedit} and AnyEdit++~\cite{tian2026anyeditpp} segment the edit sample---by a fixed sliding window, sentence boundaries, or peaks of Bayesian Surprise---and use the last token of each chunk as the representative edit key $k$.
However, none of these strategies considers the geometric position of $k$ within $\mathcal{M}_{\text{in}}$. The last-token key aggregates context from a broad mixture of preceding tokens, whose representations are distributed across the high-energy directions of $\Sigma$. This mixing effect may pull $k$ toward the high-density interior of $\mathcal{M}_{\text{in}}$---the manifold centroid---producing a ``mediocre point'' that carries little discriminative information. Modifying the mapping of such keys requires a large cost (Section~\ref{sec:theoretical_advantages}), distorting $\mathcal{M}_{\text{out}}$ and degrading the editing precision and effectiveness (Section~\ref{sec:ablation}).

\paragraph{Leverage Score for Pivot Localization.}
To address the mediocre-point dilemma, we introduce the leverage score to quantify how far a key $k$ lies from the high-density interior:
\begin{equation}
S_{\text{leverage}}(k) = k^\top \Sigma^{-1} k,
\end{equation}
where a higher leverage score indicates greater peripherality and richer discriminability. We segment the edit sample based on high-leverage keys, defined as semantic pivots, thereby anchoring the edit sub-manifold more effectively. Such semantic pivots require only a small cost to realize the displacement of $\mathcal{M}^{\text{edit}}_{\text{out}}$, which minimizes disruption to the global manifold structure, enhancing both update effectiveness and efficiency.

\begin{figure}[t]
  \centering
  \adjustbox{center, width=0.80\textwidth}{%
  \includegraphics[width=\textwidth]{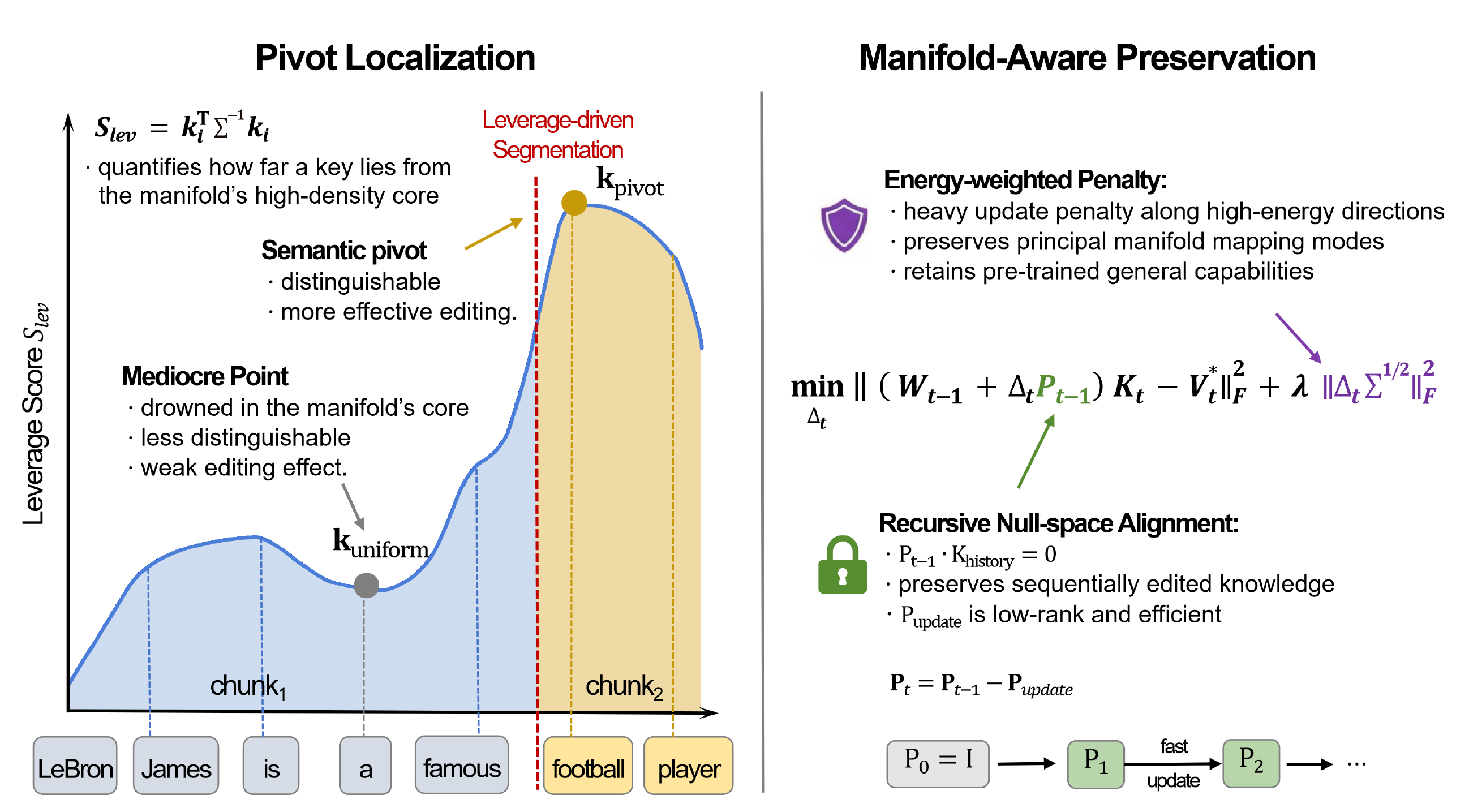}%
  }
  \vspace{-5pt}
  \caption{Overview of the ManiEdit paradigm. Pivot Localization identifies high-leverage semantic pivots to guide the segmentation of input samples, then the Manifold-Aware Preservation mechanism applies energy-weighted penalty and recursive null-space alignment to preserve general capabilities and historical edits.}
  \label{fig:ManiEdit_overview}
  \vspace{-13pt}
\end{figure}

\subsection{Manifold-Aware Preservation}
\label{sec:ssa}
Having selected the semantic pivots, we now instantiate $\mathcal{E}$ at editing step $t$ by the additive update $W_t = W_{t-1} + \Delta_t$, and derive constraints on the perturbation $\Delta_t$ that preserve these two geometrically distinct types of knowledge.

\paragraph{Preserving the Geometric Structure of the Manifold.}
The retain sub-manifold $\mathcal{M}^{\text{retain}}_{\text{in},t}$ is a globally distributed continuous structure whose geometry is encoded in the principal axes $\{\psi_j\}$ and energies $\{\lambda_j\}$. Any weight perturbation $\Delta_t$ acts on $\mathcal{M}_{\text{in}}$ as a mapping distortion, with the magnitude along $\psi_j$ given by $\|\Delta_t \psi_j\|_2$. To preserve the manifold geometry, distortions along high-energy directions should be penalized more strongly, as these directions capture dominant pre-trained geometric structure. Accordingly, we minimize the energy-weighted penalty:
\begin{equation}
\sum_{j=1}^{n} \lambda_j \| \Delta_t \psi_j \|_2^2 = \| \Delta_t \Sigma^{1/2} \|_F^2.
\end{equation}
This yields a soft geometric constraint that preserves the pre-trained manifold structure without rigidly freezing the entire representation space.

\paragraph{Disentangling Sequentially Displaced Sub-Manifolds.}
Each historical edit $i < t$ has already established a localized manifold displacement: the edit sub-manifold $\mathcal{M}^{\text{edit}}_{\text{out},i}$ was repositioned to its target. A new update $\Delta_t$ should not disturb these established displacements. A soft constraint alone is fundamentally insufficient for this and cannot prevent subspace overlap between the current edit keys $K_t$ and prior edit keys $\{K_i\}_{i<t}$, risking corruption of previously established sub-manifold locations. 
To effectively disentangle sequentially displaced sub-manifolds, we apply a recursive null-space alignment $P_{t-1}$, which satisfies $P_{t-1} K_i = 0$ for all $i < t$, guiding $\Delta_t$ toward the orthogonal complement of $\mathrm{span}\{K_i\}_{i<t}$.

\subsection{The Unified Optimization Objective} 
\label{sec:objective}
After establishing Pivot Localization and Manifold-Aware Preservation, we synthesize them into a unified optimization objective. 
At editing step $t$, the columns of $K_t$ are the semantic-pivot keys sampled from $\mathcal{M}^{\text{edit}}_{\text{in},t}$, and $V_t^*$ is the corresponding target value matrix specifying the desired location of $\mathcal{M}^{*,\text{edit}}_{\text{out},t}$. The residual displacement that $\Delta_t$ must realize is $R_t = V_t^* - W_{t-1} K_t$. We seek to minimize the displacement error of the \emph{recursive-null-space-alignment} update, subject to the \emph{energy-weighted} penalty:
\begin{equation}
\label{eq:objective} 
\min_{\Delta_t} \bigl\| (W_{t-1} + \Delta_t P_{t-1}) K_t - V_t^* \bigr\|_F^2 + \lambda \| \Delta_t \Sigma^{1/2} \|_F^2.
\end{equation}
The first term guides the effective update toward the orthogonal complement of $\mathrm{span}\{K_i\}_{i<t}$, keeping the displacement of $\mathcal{M}^{\text{edit}}_{\text{out},t}$ from perturbing previously established sub-manifolds. The second term introduces a manifold-geometry-weighted penalty that suppresses distortion of $\mathcal{M}_{\text{in}}$ along high-energy directions, preserving the global manifold geometry. $W_{t-1}$ is the current weight matrix and $\lambda$ balances the two constraints.

This objective admits a unique global closed-form solution (see Appendix~\ref{app:derivation} for the full derivation):
\begin{equation}
\label{eq:closed_form}
\Delta_t = R_t K_t^\top P_{t-1} \left( P_{t-1} K_t K_t^\top P_{t-1} + \lambda \Sigma \right)^{-1}.
\end{equation}
We then update $W_t = W_{t-1} + \Delta_t$ (see Appendix~\ref{app:leakage} for a justification).

To incorporate the newly displaced sub-manifold and protect it in subsequent edits, the recursive null-space alignment matrix is updated by excluding the subspace spanned by $K_t$:
\begin{equation}
\label{eq:recursive_null_space_alignment}
P_{t} = P_{t-1} - P_{t-1} K_t \left(K_t^\top P_{t-1} K_t + \epsilon I\right)^{-1} K_t^\top P_{t-1},
\end{equation}
where $P_0 = I$ and $\epsilon$ is a numerical stability constant. 
This recursive update accumulates all edit sub-manifolds into the protected subspace, thus effectively preserving every historical edit.

\subsection{Theoretical Advantages}
\label{sec:theoretical_advantages}

To unveil how the components of ManiEdit collaborate to realize the model editing goal of Equation~\ref{eq:edit_goal}, we present the following three propositions (see Appendix~\ref{app:spectral_derivation} for full derivation).

\begin{proposition}
  \label{prop:leverage}
  To achieve a fixed displacement $\Delta k_t = R$, the minimum cost is
  $\|\Delta \Sigma^{1/2}\|_F^2 = \|R\|^2 / (k_t^\top \Sigma^{-1} k_t)$,
  which is strictly decreasing in the leverage score $k_t^\top \Sigma^{-1} k_t$.
\end{proposition}
This explains how high-leverage semantic pivots improve the effectiveness of each weight update: they require only a small cost (little disruption to global manifold structure) to realize the displacement of $\mathcal{M}^{\text{edit}}_{\text{out},t}$, reducing unintended perturbations.

\begin{proposition}
  \label{prop:soft}
  Consider the single-key case $K_t = k_t$, and let $\hat{k}_t = P_{t-1} k_t$ and $\tilde{k}_j = \psi_j^\top \hat{k}_t$. Then the closed-form update in Equation~\ref{eq:closed_form} satisfies $\Delta W_{\text{dir}\ j} := \Delta_t \psi_j \propto \tilde{k}_j / \lambda_j$.
  \end{proposition}
The numerator $\tilde{k}_j$ measures the semantic demand of the edit along $\psi_j$, while the denominator $\lambda_j$ acts as an energy-adaptive penalty: directions that genuinely need updating receive capacity, whereas perturbations along dominant pre-trained directions are suppressed.

\begin{proposition}
  \label{prop:hard} 
  Set $\epsilon = 0$ in Equation~\ref{eq:recursive_null_space_alignment} and replace the inverse with the pseudo-inverse. Then for every $t \ge 1$: (i) $P_t K_i = 0$ for all $i \le t$; (ii) $\mathrm{rank}(P_t) \ge n - \sum_{i=1}^{t} \mathrm{rank}(K_i)$.
  \end{proposition}
Unlike traditional null-space methods whose $P$ enforces $\Delta W_{\text{dir } j} = 0$ for $\lambda_j > \theta$ (e.g., AlphaEdit~\cite{fang2025alphaedit}, EvoEdit~\cite{lyu2025evoedit}), which severely compresses the feasible-update space and impairs editing efficacy, our $P_{t}$ allows energy-adaptive weight updates that preserve historical edits without collapsing the feasible-update rank.

Putting these components together:
Pivot Localization first identifies easily displaceable regions on $\mathcal{M}^{\text{edit}}_{\text{in}}$, enabling precise editing at low cost. 
During the displacement, the energy-weighted penalty protects the global manifold structure from unintended distortion. 
Finally, recursive null-space alignment stabilizes the relocated edit sub-manifold. 
This joint mechanism ensures effective editing while preserving historical edits and general capabilities.

\begin{table}[t]
  \centering
  \footnotesize
  \setlength{\tabcolsep}{1.8pt}
  \resizebox{\textwidth}{!}{%
  \begin{tabular}{@{}l *{12}{c}@{}}
  \toprule
  \multirow{3}{*}{\textbf{Method}}
    & \multicolumn{4}{c}{\textbf{UnKEBench}}
    & \multicolumn{4}{c}{\textbf{AKEW (Counterfact)}}
    & \multicolumn{4}{c}{\textbf{AKEW (MQUAKE)}} \\
  \cmidrule(lr){2-5}\cmidrule(lr){6-9}\cmidrule(lr){10-13}
    & \multicolumn{2}{c}{Ori} & \multicolumn{2}{c}{Para}
    & \multicolumn{2}{c}{Ori} & \multicolumn{2}{c}{Para}
    & \multicolumn{2}{c}{Ori} & \multicolumn{2}{c}{Para} \\
  \cmidrule(lr){2-3}\cmidrule(lr){4-5}\cmidrule(lr){6-7}\cmidrule(lr){8-9}\cmidrule(lr){10-11}\cmidrule(lr){12-13}
    & BERTScore & ROUGE-L & BERTScore & ROUGE-L & BERTScore & ROUGE-L & BERTScore & ROUGE-L & BERTScore & ROUGE-L & BERTScore & ROUGE-L \\
  \midrule
  \multicolumn{13}{@{}l}{\textbf{Based on Qwen2.5-7B-Instruct}} \\
  \midrule
  MEMIT          & 8.36  & 16.89 & 8.95  & 16.73 & 9.07  & 19.23 & 9.43  & 19.22 & 10.47 & 18.71 & 9.17  & 18.67 \\
  RECT           & 12.28 & 17.57 & 12.35 & 17.07 & 6.21  & 17.08 & 5.90  & 16.76 & 13.34 & 21.90 & 13.20 & 21.37 \\
  EvoEdit        & 25.79 & 16.08 & 21.60 & 15.16 & 35.89 & 24.56 & 19.85 & 20.95 & 52.28 & 26.35 & 44.51 & 21.76 \\
  UnKE           & 9.45  & 8.12  & 12.78 & 10.55  & 5.41  & 9.51  & 4.28  & 10.41  & 12.69  & 18.37 & 10.46  & 17.28  \\
  AnyEdit        & 9.80  & 15.94 & 9.67  & 15.49 & 1.06  & 13.66 & 0.75  & 13.38 & 7.80  & 15.02 & 6.88  & 16.03 \\
  AnyEdit++      & 21.90 & 13.21  & 21.61 & 13.15  & 37.42  & 22.32  & 28.51  & 21.10  & 49.24  & 30.40 & 41.65  & 23.97 \\
  FT-UKE         & 34.82 & 21.28  & 34.18 & 20.36  & 37.16  & 21.87  & 18.17  & 15.28  & 52.73  & 34.75 & 46.95  & 29.85 \\
  MOSE           & \uline{75.13} & \uline{31.72} & \textbf{75.43} & \uline{30.87} & \uline{71.03} & \uline{29.67} & \textbf{53.98} & \uline{24.54} & \uline{66.05} & \uline{34.87} & \uline{66.14} & \uline{32.47} \\
  \cmidrule(lr){1-13}
  ManiEdit      & \textbf{76.63} & \textbf{39.43} & \uline{74.97} & \textbf{38.07} & \textbf{73.88} & \textbf{38.17} & \uline{50.78} & \textbf{31.45} & \textbf{71.37} & \textbf{37.96} & \textbf{70.87} & \textbf{35.08} \\
  \midrule
  \multicolumn{13}{@{}l}{\textbf{Based on Llama3-8B-Instruct}} \\
  \midrule
  MEMIT          & 2.00  & 1.18  & 2.44  & 1.05  & 1.58  & 6.15  & 1.60  & 6.37  & 0.97  & 5.81  & 1.53  & 6.10  \\
  RECT           & 10.91 & 16.69 & 11.37 & 16.73 & 2.00  & 10.27 & 1.97  & 10.26 & 0.00  & 8.88  & 0.00  & 8.71  \\
  EvoEdit        & 1.67  & 9.77  & 1.52  & 9.91  & 19.20 & 12.37 & 7.45  & 9.82  & 23.58 & 9.89  & 24.63 & 9.87  \\
  UnKE           & 1.26  & 1.12  & 1.76  & 1.20  & 0.61  & 3.33  & 0.59  & 3.49  & 2.49  & 5.04  & 0.10  & 2.83  \\
  AnyEdit        & 11.55 & 14.40 & 11.40 & 13.95 & 4.40  & 8.55  & 5.11  & 9.17  & 8.02  & 14.19 & 8.70  & 13.86 \\
  MOSE           & \uline{62.38} & \uline{31.76} & \uline{61.11} & \uline{30.95} & \uline{50.69} & \uline{30.65} & \uline{43.77} & \uline{24.08} & \uline{63.48} & \textbf{37.97} & \uline{63.57} & \textbf{36.98} \\
  \cmidrule(lr){1-13}
  ManiEdit      & \textbf{76.22} & \textbf{37.30} & \textbf{75.80} & \textbf{35.61} & \textbf{78.50} & \textbf{37.24} & \textbf{53.88} & \textbf{28.73} & \textbf{71.51} & \uline{30.03} & \textbf{71.24} & \uline{35.34} \\
  \bottomrule
  \end{tabular}%
  }%
  \caption{Comparison of ManiEdit with existing methods on the sequential unstructured model editing tasks. The best results are highlighted in bold, while the second-best results are underlined.}
  \label{tab:main_results}
  \vspace{-15pt}
\end{table}

\section{Experiments}
\label{sec:experiments}

To rigorously evaluate the effectiveness of ManiEdit, we design a comprehensive set of experiments focusing on sequential long-form knowledge editing. Our evaluation seeks to answer the following research questions:
\begin{itemize}[leftmargin=*,nosep]
    \item \textbf{RQ1 (Overall Performance):} How does ManiEdit perform on sequential unstructured editing tasks compared to baseline methods? Can it mitigate the issue of edit forgetting?
    \item \textbf{RQ2 (General Capability Test):} How well does ManiEdit maintain general capabilities after hundreds of continual edits? Can it prevent the collapse of general capabilities?
    \item \textbf{RQ3 (Analysis of Pivot Localization):} Can Pivot Localization consistently improve the editing effectiveness of ManiEdit and other methods?
\end{itemize}
Due to space limitations, further ablations of Manifold-Aware Preservation, long-horizon stress tests, and other extended experiments are provided in Appendix~\ref{app:extended_experiments}.

\subsection{Experimental Setup}
In this subsection, we briefly outline the base LLMs, baseline methods, datasets, and evaluation metrics. More details are provided in Appendix~\ref{app:exp_details}.

\vspace{-5pt}
\paragraph{Base LLMs \& Baseline Methods.} We experiment on two widely adopted instruction-tuned LLMs: Llama3-8B-Instruct~\cite{grattafiori2024llama3} and Qwen2.5-7B-Instruct~\cite{qwen2025qwen25}. We compare ManiEdit against a comprehensive suite of model editing baselines spanning multiple paradigms: (1) \textit{Traditional editing methods}: MEMIT~\cite{meng2023mass} and RECT~\cite{gu2024rect}; (2) \textit{Sequential editing methods}: EvoEdit~\cite{lyu2025evoedit} and MOSE~\cite{xu2026multiplicative_orthogonal}; (3) \textit{Unstructured editing methods}: UnKE~\cite{deng2025everything}, FT-UKE~\cite{xiong2025finetuning}, AnyEdit~\cite{jiang2025anyedit}, and AnyEdit++~\cite{tian2026anyeditpp}. We further include three memory-based editing methods (i.e., SERAC~\cite{mitchell2022memory}, GRACE~\cite{hartvigsen2023aging}, and MELO~\cite{yu2023melo}) as additional baselines in Appendix~\ref{app:extended_splits}.

\vspace{-5pt}
\paragraph{Datasets \& Evaluation Metrics.} We evaluate ManiEdit on four established benchmarks for sequential unstructured knowledge editing: UnKEBench~\cite{deng2025everything}, AKEW (Counterfact)~\cite{wu2024akew}, AKEW (MQUAKE)~\cite{wu2024akew}, and EditEverything~\cite{jiang2025anyedit}. We evaluate the similarity between the edited model's outputs and the editing targets under three query conditions: Original (Ori), Paraphrase (Para), and Sub-question (Sub). The evaluation uses BERTScore~\cite{zhang2019bertscore} for semantic similarity and ROUGE-L~\cite{lin2004rouge} for lexical overlap.

\begin{figure}[t]
  \centering
  \includegraphics[width=1\linewidth]{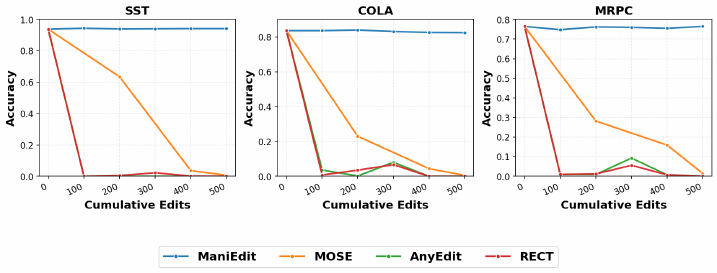}
  \caption{General capability accuracy on three GLUE tasks (SST-2, CoLA, MRPC) as the number of sequential unstructured edits increases. ManiEdit maintains near-original accuracy across all tasks, while baselines collapse rapidly.}
  \label{fig:glue_accuracy}
  \vspace{-15pt}
\end{figure}

\subsection{Sequential Unstructured Editing (RQ1)}

To evaluate the performance of different editing methods on sequential unstructured knowledge editing, we conduct experiments on two base LLMs and four benchmarks, comparing ManiEdit with all baselines. 
Table~\ref{tab:main_results} presents the results under a commonly used configuration for the sequential editing task, where 500 samples are randomly drawn from the dataset for updates, with one sample per edit (i.e., a batch size of 1). Since AKEW (MQUAKE) contains only 354 samples in total, we evaluate it on the full set rather than sampling 500. 
More sub-question-level evaluations for each benchmark are reported in Appendix~\ref{app:sub_parametric}.
Results on EditEverything are provided in Appendix~\ref{app:editeverything}.
Based on Table~\ref{tab:main_results}, we draw the following observations:

\paragraph{ManiEdit achieves the best overall performance across all benchmarks and LLMs.} Compared with the strongest baseline, ManiEdit improves on 20 out of 24 metric columns. The largest gains reach \textbf{+27.81} BERTScore on Llama3-8B-Instruct AKEW (Counterfact) Ori and \textbf{+8.50} ROUGE-L on Qwen2.5-7B-Instruct AKEW (Counterfact) Ori. Averaged over all settings, ManiEdit improves upon the strongest baseline by +7.74 BERTScore and +3.99 ROUGE-L. These gains arise from ManiEdit's precise localization of semantic pivots and its protection of historical edit sub-manifolds.

\subsection{General Capability Test (RQ2)}

To evaluate whether post-edited LLMs retain their intrinsic general capabilities, we perform sequential unstructured editing on Qwen2.5-7B-Instruct for 500 edits and track the downstream task accuracy at intervals of 100 edits. 
Figure~\ref{fig:glue_accuracy} reports the accuracy curves for ManiEdit against three competitive baselines---MOSE, AnyEdit, and RECT---evaluated on three representative GLUE tasks:
SST-2~\cite{socher2013sst} (sentiment classification), CoLA~\cite{warstadt2019cola} (linguistic acceptability), and MRPC~\cite{dolan2005mrpc} (paraphrase detection).
Further experiments on MMLU~\cite{hendrycks2021mmlu}, MNLI~\cite{williams2018mnli}, and RTE~\cite{bentivogli2009rte} are reported in Appendix~\ref{app:extended_glue}.
Based on Figure~\ref{fig:glue_accuracy}, we draw the following observations:

\paragraph{Baseline methods suffer catastrophic collapse of general capabilities within 100--200 edits.} AnyEdit and RECT drop to near-zero accuracy on all three tasks after merely 100 edits, indicating that their parameter updates severely damage the model's general language ability. 
MOSE exhibits a slower but still fatal decline---suffering substantial accuracy degradation across all three tasks by 200 edits and ultimately collapsing to near-zero levels by 500 edits.
This confirms that existing methods struggle to prevent general-capability collapse.

\paragraph{ManiEdit sustains near-original general capabilities even after 500 edits.} Across all three tasks, ManiEdit maintains accuracy within 1--2\% of the unedited model throughout the entire 500-edit sequence (e.g., SST-2: 93.8\% $\to$ 94.2\%; CoLA: 83.6\% $\to$ 82.5\%; MRPC: 76.5\% $\to$ 76.5\%). 
This stability arises from ManiEdit's protection of the geometry of the pre-trained manifold.

\begin{figure}[t]
  \centering
  \includegraphics[width=1\linewidth]{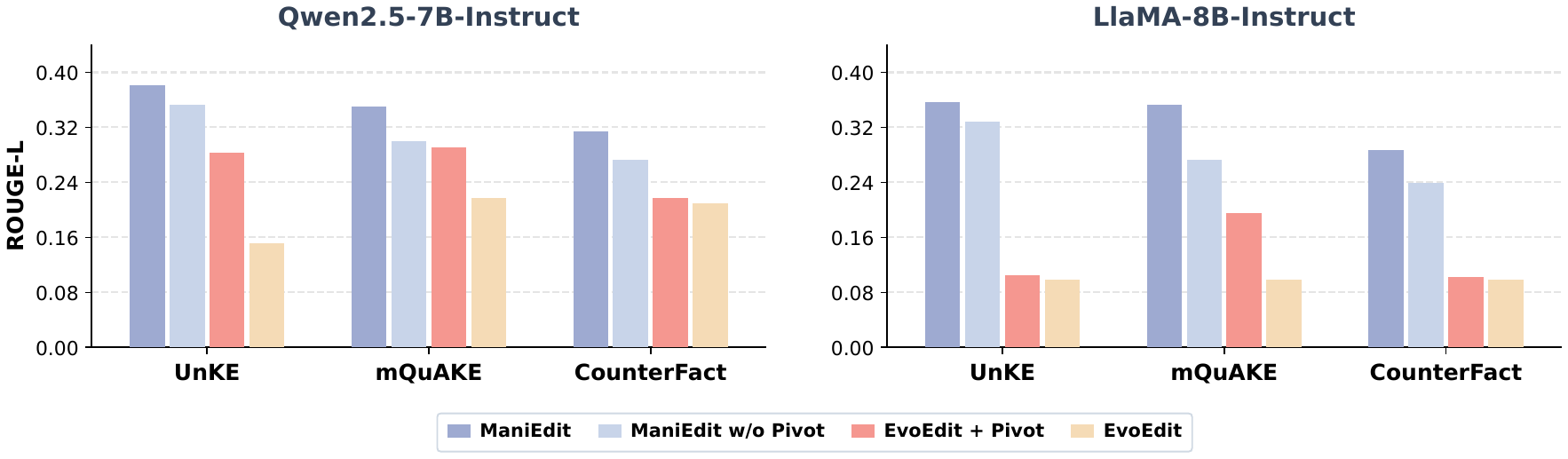}
  \caption{Analysis of Pivot Localization. Scores are Para ROUGE-L on two base LLMs across three benchmarks. ManiEdit w/o Pivot and EvoEdit use AnyEdit's fixed-window chunking.}
  \label{fig:ablation_para}
  \vspace{-15pt}
\end{figure}

\subsection{Analysis of Pivot Localization (RQ3)}
\label{sec:ablation}
To examine whether Pivot Localization provides a general benefit beyond ManiEdit itself, we evaluate it under two complementary settings: removing Pivot Localization from ManiEdit, and adding it to EvoEdit, a representative locate-then-edit baseline. 
The experimental setting follows Table~\ref{tab:main_results}.
Figure~\ref{fig:ablation_para} reports Para ROUGE-L across three benchmarks on two base LLMs. Additional Ori ROUGE-L results are provided in Appendix~\ref{app:ablation_ori}. We draw the following observation based on Figure~\ref{fig:ablation_para}:

\paragraph{Pivot Localization consistently improves editing effectiveness.} 
For EvoEdit, adding Pivot Localization improves ROUGE-L in all six settings, with large gains on Qwen-UnKEBench ($+13.13$), Qwen-MQUAKE ($+7.37$), and Llama-MQUAKE ($+9.68$). This indicates that semantic pivots are not tied to ManiEdit's preservation module, but can also strengthen other methods. 
For ManiEdit, removing Pivot Localization lowers ROUGE-L in every setting; the full model improves over the no-pivot variant by $+4.62$ ROUGE-L on average, ranging from $+2.82$ on Llama-UnKEBench to $+8.10$ on Llama-MQUAKE. 
These gains confirm that Pivot Localization improves the editing effectiveness of ManiEdit and other methods consistently.

\vspace{-3pt}

\section{Related Work}
\label{sec:related_work}

\vspace{-3pt}

This section summarizes related work; a comprehensive discussion is presented in Appendix~\ref{app:related_work}.

\vspace{-3pt}

\paragraph{Knowledge Editing.}
Traditional methods modify targeted facts represented as triplets. \textit{Parameter-preserving} methods utilize additional modules to store to-be-updated knowledge.
These modules may include codebooks, matrices, or auxiliary models, as seen in methods like GRACE~\cite{hartvigsen2023aging}, MELO~\cite{yu2023melo}, and SERAC~\cite{mitchell2022memory}.
\textit{Parameter-modifying} approaches permanently integrate knowledge into the base weights. Hyper-network methods such as MEND~\cite{mitchell2022fast} train auxiliary networks to predict weight updates. The Locate-then-Edit paradigm, exemplified by MEMIT~\cite{meng2023mass}, localizes critical knowledge neurons via causal tracing and applies closed-form updates. AlphaEdit~\cite{fang2025alphaedit} and EvoEdit~\cite{lyu2025evoedit} employ null-space projection to mitigate decay under sequential editing. 
OTE-SE~\cite{wang2026labyrinth} keeps the accumulated sequential update equivalent to the corresponding one-time editing solution. 
MOSE~\cite{xu2026multiplicative_orthogonal} proposes multiplicative orthogonal updates as an alternative to standard additive weight increments.

\vspace{-3pt}

\paragraph{Unstructured Knowledge Editing.} 
Real-world knowledge is typically embedded in unstructured, long-form text with rich contextual dependencies. Recent work has begun to address this more challenging setting. UnKE~\cite{deng2025everything} treats the model as a key-encoder and value-generator and trains critical layers. FT-UKE~\cite{xiong2025finetuning} provides a training recipe that makes vanilla fine-tuning effective for unstructured knowledge editing. 
AnyEdit~\cite{jiang2025anyedit} decomposes long-form text into uniform chunks and iteratively edits the last token of each chunk.
AnyEdit++~\cite{tian2026anyeditpp} instead places chunk boundaries at peaks of Bayesian Surprise.
\vspace{-3pt}
\section{Conclusion}
\vspace{-3pt}
In this paper, we present ManiEdit for sequential unstructured knowledge editing. By novelly reframing editing as a localized sub-manifold displacement, we propose Pivot Localization and Manifold-Aware Preservation for edit point identification and manifold structure preservation. Experiments confirm that ManiEdit not only achieves superior editing performance but also effectively prevents edit forgetting and general-capability collapse, marking a significant advance toward the real-world deployment of knowledge editing. 

\subsection*{AI use statement}

In this work, we used generative AI tools to edit the manuscript for readability, to format references, and to assist in the writing of proofs.
We have not used generative AI tools to help develop theoretical models or conceptual frameworks, formulate mathematical claims, provide critical ingredients for proving mathematical claims, propose or refine hypotheses, design or provide feedback on research methodology or experiments, implement methods, or interpret results. 
Generating synthetic datasets, assisting with translation, cleaning and reformatting datasets, and supporting qualitative and thematic data analysis are not applicable to this work.
We have reviewed all AI-assisted work: LLM-assisted text and proof write-ups were checked and edited by the authors for factual accuracy.
We take responsibility for the final content of this work, including text, claims or artifacts produced with the aid of generative AI.

\subsection*{Reproducibility statement}

We have taken the following steps to facilitate reproducibility.
The ManiEdit algorithm, including Pivot Localization and the closed-form Manifold-Aware Preservation update, is specified in Section~\ref{sec:method}.
Complete proofs of the theoretical claims are given in Appendices~\ref{app:derivation} and~\ref{app:spectral_derivation}.
Backbone models, baselines, hyperparameters, datasets, evaluation metrics, and hardware settings are reported in Appendix~\ref{app:exp_details}.
All editing and general-capability benchmarks used in this work are publicly available.

\bibliography{custom}
\bibliographystyle{maniedit_references}

\clearpage
\appendix
\section{Extended Preliminaries}
\label{app:prelim}

After the critical FFN layer is identified (typically via causal tracing), standard locate-then-edit methods first optimize a target value $v^{*}$ for a chosen key $k$ and then write the pair $(k, v^{*})$ into $W$ (Section~\ref{sec:preliminary}). We recap this procedure for structured facts, and then its extension to unstructured long-form knowledge.

\subsection{Standard Locate-then-Edit Procedure}
A structured fact is represented as a subject--relation--object triplet $(s,r,o)$. ROME~\cite{meng2023locating} and MEMIT~\cite{meng2023mass} encode it as a key and a corresponding value: the key is the FFN activation at the last token of the subject $s$, and the value is the associated FFN output that should realize the object $o$.

To obtain this value, a residual $\delta$ is added to the current FFN output $h = Wk$ and optimized to raise the likelihood of $o$:
\begin{equation}
\label{eq:vec_opt}
\delta^{*} = \operatorname*{arg\max}_{\delta} \log P\bigl(o \mid s, r;\, h \leftarrow h + \delta\bigr).
\end{equation}
The target value is then $v^{*} = h + \delta^{*}$.

The resulting pair $(k, v^{*})$ is then written into $W$ by a constrained least-squares update that maps $k$ to $v^{*}$ while limiting drift on other keys~\cite{meng2023locating,meng2023mass}. 
Although this procedure is effective for short facts, a single pair $(k, v^{*})$ is often insufficient for long-form text~\cite{jiang2025anyedit}.

\subsection{Autoregressive Editing of Long-Form Knowledge}
To edit unstructured long-form text $Y$, methods such as AnyEdit~\cite{jiang2025anyedit} segment $Y$ into chunks $\{C_j\}_{j=1}^{M}$ and apply vector optimization to each chunk in order. For chunk $C_j$, the key $k_j$ is the FFN activation at the last token of the preceding chunk. A local residual $\delta_j$ is optimized to generate $C_j$ given the prefix and previous chunks, while the residuals $\{\delta_i^{*}\}_{i<j}$ remain active in the forward pass:
\begin{equation}
\label{eq:ar_delta}
\delta_j^{*} = \operatorname*{arg\max}_{\delta_j} \log P\bigl(C_j \mid \text{prefix},\, C_{<j};\,
\{h_i \leftarrow Wk_i + \delta_i^{*}\}_{i<j},\; h_j \leftarrow Wk_j + \delta_j\bigr).
\end{equation}
Setting $v_j^{*} = Wk_j + \delta_j$ yields the pairs $(k_j, v_j^{*})$. After all pairs are collected, $\{(k_j, v_j^{*})\}_{j=1}^{M}$ are then written into $W$ in a single batch.

\section{Experimental Setup Details}
\label{app:exp_details}

\subsection{Baseline Methods}
\label{app:baselines}

We compare ManiEdit against a diverse set of knowledge editing baselines spanning three major paradigms. Below we provide a detailed description of each method.

\paragraph{Knowledge Editing Methods.}
These methods edit structured knowledge represented as subject--relation--object triplets. The dominant locate-then-edit paradigm identifies critical parameters (typically FFN weights) associated with the target fact and applies optimization-based closed-form updates.

\begin{itemize}
\item \textbf{MEMIT}~\cite{meng2023mass} extends ROME to mass-editing by distributing parameter updates across multiple MLP layers via closed-form solutions. While effective for batch editing of structured triplets, it lacks mechanisms to prevent inter-edit interference in sequential settings, leading to edit forgetting. 

\item \textbf{RECT}~\cite{gu2024rect} is a plug-in regularizer on locate-then-edit updates such as ROME and MEMIT, rather than a standalone editor. It keeps the entries of the update whose relative change is largest and sets the rest to zero, sparsifying the update so that it is less prone to overfit the edited facts.

\item \textbf{EvoEdit}~\cite{lyu2025evoedit} addresses sequential editing by projecting each new update into the null-space of previously protected knowledge, geometrically isolating successive edits. However, it applies hard projection uniformly to both pre-trained and edited knowledge, which collapses the feasible-update rank, thereby compressing the editable space and impairing editing efficacy.

\item \textbf{MOSE}~\cite{xu2026multiplicative_orthogonal} proposes multiplicative orthogonal updates as an alternative to the standard additive weight increments, aiming to preserve the norm and condition number of the parameter matrix throughout sequential editing. This geometric stabilization strategy demonstrates strong retention on structured triplet editing, but overlooks the geometric asymmetry between pre-trained knowledge and sequential edits, leading to suboptimal outcomes in long-form scenarios.
\end{itemize}

\paragraph{Unstructured Editing Methods.}
These methods are specifically designed to handle unstructured, paragraph-level or document-level knowledge.

\begin{itemize}
\item \textbf{UnKE}~\cite{deng2025everything} achieves unstructured knowledge editing by performing block-level (rather than single-layer) localization and gradient-descent-based optimization. While it achieves strong single-edit performance, its reliance on gradient descent without explicit orthogonal constraints makes it vulnerable to catastrophic interference when applied to sequential editing.

\item \textbf{FT-UKE}~\cite{xiong2025finetuning} reassesses fine-tuning for unstructured knowledge editing and provides a training recipe under which vanilla fine-tuning becomes a stronger UKE baseline. However, it likewise suffers from inter-edit interference under sequential editing, leading to edit forgetting.

\item \textbf{AnyEdit}~\cite{jiang2025anyedit} introduces a sequential chunking strategy that decomposes long-form text into fixed-length segments and iteratively applies AlphaEdit-style updates per chunk. This autoregressive editing paradigm enables processing of arbitrarily long texts but falls into the mediocre-point dilemma, degrading the editing precision and effectiveness.

\item \textbf{AnyEdit++}~\cite{tian2026anyeditpp} replaces uniform windows with adaptive chunk boundaries at peaks of Bayesian Surprise. While this segmentation better respects semantic transitions, it still ignores the geometric position of the edit key within $\mathcal{M}_{\text{in}}$ and thus remains subject to the mediocre-point dilemma.

\end{itemize}

\paragraph{Memory-based Methods.}
These methods utilize additional modules to store to-be-updated knowledge. Although they avoid modifying the original model weights, they rely heavily on retrieval accuracy and struggle with generalized queries that paraphrase edited facts. Furthermore, their reliance on external storage raises concerns regarding scalability and portability.

\begin{itemize}
\item \textbf{SERAC}~\cite{mitchell2022memory} adopts a semi-parametric architecture: edits are stored in an external memory bank, and a trained scope classifier determines at inference time whether a given input should be routed to an auxiliary counterfactual model or the original base model. 

\item \textbf{GRACE}~\cite{hartvigsen2023aging} maintains a discrete key-value codebook at a designated model layer. During inference, the input's hidden representation is matched against codebook keys; if a match is found, the corresponding learned value vector replaces the original layer output. This design supports lifelong editing with inter-edit isolation, but effectiveness depends on the quality of the matching mechanism.

\item \textbf{MELO}~\cite{yu2023melo} assigns non-overlapping LoRA adaptor blocks to different edits, with a routing mechanism that activates the relevant adaptor at inference time. It introduces additional parameter modules, therefore incurring growing storage and routing overhead as edits accumulate.

\end{itemize}

\subsection{Datasets}
\label{app:datasets}

We evaluate ManiEdit on four benchmarks that collectively cover unstructured long-form knowledge, counterfactual updates, multi-hop reasoning consistency, and diverse-domain editing.

\paragraph{UnKEBench.} UnKEBench~\cite{deng2025everything} is constructed to evaluate the editing of unstructured, relatively lengthy knowledge that goes beyond simple triplets or linear fact chains. It contains 1,000 counterfactual unstructured texts originating from ConflictQA, a benchmark specifically designed to distinguish LLMs' parametric memory from counter-memory. This design is crucial for ensuring a clear boundary between pre-training knowledge and edited knowledge, preventing the model from conflating the two. Each sample includes a main question (with a paraphrase variant) and a reference long-form answer as the editing target. Additionally, sub-question lists are provided for every sample for multi-angle testing of the edited knowledge.

\paragraph{AKEW (Counterfact).} The AKEW benchmark~\cite{wu2024akew} (Assessing Knowledge Editing in the Wild) provides a more realistic evaluation framework by supporting three input modalities for the same knowledge update: structured facts (isolated triplets), unstructured facts (natural language descriptions), and extracted triplets (automatically derived from unstructured text). The Counterfact subset contains 975 samples and focuses on counterfactual updates, where each sample contains a \texttt{requested\_rewrite} specifying the subject, relation, and target replacement, along with both a structured one-sentence fact (\texttt{fact\_new}) and a richer unstructured description (\texttt{fact\_new\_uns}). In this work, we primarily focus on the unstructured factual knowledge setting to align with our long-form editing objective.

\paragraph{AKEW (MQUAKE).} The MQUAKE-CF subset within AKEW contains 354 samples and extends the evaluation to multi-hop reasoning consistency. Derived from the MQUAKE benchmark, it examines whether editing a base fact triggers the correct cascading changes in answers to multi-hop questions that logically depend on the edited fact. AKEW frames this multi-hop verification under the same three-modality protocol (structured / unstructured / extracted), enabling systematic comparison of how different input representations affect editing fidelity across reasoning chains. The sample fields mirror those of AKEW-Counterfact, with the key distinction being the emphasis on entailed multi-hop consequences rather than single-fact recall.

\paragraph{EditEverything.} EditEverything~\cite{jiang2025anyedit} integrates question-answering data from multiple specialized domains, constructing long and diverse knowledge formats that are substantially more challenging to edit. Specifically, the benchmark includes: (1) \textit{Mathematics}: longer samples from the Orca-Math dataset containing grade-school math word problems; (2) \textit{Code}: samples from the MBPP dataset, comprising crowd-sourced Python programming problems covering programming fundamentals and standard library functionalities; (3) \textit{Chemistry}: problem-solution pairs from the Camel-Chemistry dataset, spanning 25 chemistry topics with structured subtopic hierarchies; (4) \textit{Biology}: scientific QA pairs drawn from domain-specific corpora; (5) \textit{News} and \textit{Poetry}: since these categories often contain real-world knowledge that LLMs may already possess, synthetic data is generated using GPT-4o to ensure the target information is novel to the model. This cross-domain design tests whether editing methods can generalize beyond standard factual text to heterogeneous knowledge formats.

\subsection{General Capability Benchmarks}
\label{app:glue_datasets}

Beyond the long-form knowledge editing benchmarks above, we evaluate general capability preservation using six established benchmarks.

\paragraph{SST-2.} The Stanford Sentiment Treebank~\cite{socher2013sst} is a single-sentence binary sentiment classification task built from movie review sentences with human-annotated sentiment labels. It is widely used to probe whether edited models retain surface-level language understanding.

\paragraph{MRPC.} The Microsoft Research Paraphrase Corpus~\cite{dolan2005mrpc} is a sentence-pair benchmark for semantic equivalence detection. Each sample consists of two sentences, and the model must determine whether they are paraphrases. MRPC tests whether parameter edits preserve fine-grained semantic reasoning.

\paragraph{CoLA.} The Corpus of Linguistic Acceptability~\cite{warstadt2019cola} is a single-sentence classification task in which sentences are annotated as grammatically acceptable or unacceptable. It probes the model's internalized syntactic knowledge after editing.

\paragraph{MMLU.} The Massive Multitask Language Understanding benchmark~\cite{hendrycks2021mmlu} is a comprehensive evaluation that measures multi-task accuracy across 57 academic subjects, spanning STEM, humanities, and social sciences. It is evaluated under the zero-shot setting and provides a broad signal of whether general world knowledge is preserved after sequential editing.

\paragraph{MNLI.} The Multi-Genre Natural Language Inference corpus~\cite{williams2018mnli} covers sentence-pair natural language inference across multiple genres of written and spoken English. Given a premise and a hypothesis, the model must infer the logical relationship (entailment, neutral, or contradiction). MNLI tests whether edits disrupt the model's cross-domain reasoning.

\paragraph{RTE.} The Recognizing Textual Entailment challenge~\cite{bentivogli2009rte} is a binary natural language inference task that determines whether a premise sentence logically entails a given hypothesis. It provides a focused test of deductive reasoning under sequential parameter perturbation.

\subsection{Evaluation Metrics}
\label{app:metrics}

Traditional evaluation metrics for structured knowledge editing---Efficacy, Generalization, and Specificity---are designed around exact-match comparisons with short target answers. While well-suited for triplet-based editing, they are insufficient for evaluating long-form and diverse-formatted knowledge, where the model may generate a response that captures the essential information yet differs substantially in surface form from the reference. To address this, we follow existing benchmarks for unstructured knowledge editing and adopt evaluation metrics that accommodate the verbosity and complexity of long-form responses.

\paragraph{Lexical Similarity.} We employ \textbf{ROUGE-L}~\cite{lin2004rouge} to measure the longest common subsequence overlap between the model-generated text and the target answer. ROUGE-L captures the sequential structure of the generated response and provides insight into the surface-level accuracy of the edited content. ROUGE-1 is additionally reported in Appendix tables.

\paragraph{Semantic Similarity.} Word-level overlap alone is insufficient for capturing the nuanced understanding that a model must exhibit after editing. A generated response may convey the correct meaning using different vocabulary or sentence structure, which lexical metrics would penalize. To bridge this gap, we employ \textbf{BERTScore}~\cite{zhang2019bertscore}, which leverages contextual embeddings to compute token-level cosine similarity between the generated text and the reference. BERTScore provides a more balanced evaluation that extends beyond lexical matching, quantifying the depth of the model's semantic comprehension of the edited knowledge. Following prior work, we report the F1 variant of BERTScore throughout our experiments.

\paragraph{Query Conditions.} We evaluate editing quality under three query conditions:
\begin{itemize}
    \item \textbf{Original (Ori):} The model is queried with the original question directly associated with the edited knowledge, testing whether the edit has been successfully incorporated. Both BERTScore and ROUGE are reported.
    \item \textbf{Paraphrase (Para):} The model is queried with a semantically equivalent but differently phrased question, testing whether the edited knowledge has been genuinely internalized rather than superficially memorized. Both BERTScore and ROUGE are reported.
    \item \textbf{Sub-question (Sub):} The model is queried with fine-grained sub-questions that probe specific aspects of the edited content. We report ROUGE against the corresponding sub-answers as a lexical measure of fine-grained coverage.
\end{itemize}

\subsection{Implementation Details}
\label{app:impl_details}

All experiments are conducted on NVIDIA RTX PRO 5000 72GB Blackwell GPUs. We adopt greedy decoding, which makes generation deterministic for a fixed edited checkpoint, so we report single-run results rather than averaged statistics over multiple random seeds. For baselines, we primarily follow the default configurations provided in each method's official repository, with minor adjustments to align cross-method comparisons. Below we summarize the key hyperparameters.

\paragraph{ManiEdit.}
On Qwen2.5-7B-Instruct, we target layer 8 for editing. Text is segmented with the top-$M$ highest-leverage semantic pivots as boundaries, rather than with a fixed-length window. The regularization weight $\lambda$ is set to 15{,}000, with $\Sigma$ estimated from 100{,}000 Wikipedia samples. The value optimization uses 25 gradient steps with a learning rate of 0.5, weight decay of 0.001, and clamp-norm factor 4. The loss is applied at layer 27. The projection matrix uses $\epsilon = 10^{-6}$ for numerical stability.
On Llama3-8B-Instruct, the key differences are: $\lambda = 30{,}000$, loss layer 31. All other parameters remain the same.

\paragraph{MEMIT \& RECT.}
Both methods use the MEMIT\_ARE implementation from the AnyEdit codebase. Edits are distributed across layers [4--8] on both backbones. On Qwen, $\lambda = 15{,}000$ and window size is 50; on Llama3, $\lambda = 30{,}000$ and window size is 40. Value optimization follows the same protocol as ManiEdit (25 steps, lr 0.5, weight decay 0.001, clamp-norm factor 4). RECT additionally applies a sparsity fraction of 0.6 to regularize weight perturbations.

\paragraph{AnyEdit \& EvoEdit.}
AnyEdit extends AlphaEdit with sequential chunking. Both methods share layers [4--8], and the same value optimization and window size settings as MEMIT per backbone. AnyEdit further introduces a null-space projection threshold of 0.02 and L2 weight of 10.

\paragraph{AnyEdit++.}
On Qwen2.5-7B-Instruct, layer 8 is targeted for editing. The value optimization uses 25 gradient steps at learning rate $0.5$ (weight decay $10^{-3}$, KL factor $0.0625$, clamp-norm factor $4$), the value-loss layer is 27, and $\lambda = 15{,}000$.

\paragraph{UnKE.}
UnKE performs gradient-descent-based block-level editing on layer 7 for both backbones, with a learning rate of $2 \times 10^{-4}$ and 50 optimization steps (early-stop loss $0.1$). To constrain parameter drift, 20 external data samples (max length 256) are used as locality anchors. The value loss layer is set to 27 (Qwen) or 31 (Llama3).

\paragraph{FT-UKE.}
FT-UKE performs gradient-descent-based editing of a single FFN \texttt{down\_proj} matrix. On Qwen2.5-7B-Instruct, layer 27 is edited with a learning rate of $5 \times 10^{-4}$, and 10 optimization steps (early-stop loss $0.1$). The maximum sequence length is 1024, and gradient clipping is set to 1.0.
 
\paragraph{MOSE.}
MOSE applies multiplicative orthogonal updates. On Qwen, layers [4, 5, 6] are edited with 20 optimization steps and learning rate $5 \times 10^{-4}$. On Llama3, layers [19, 20, 21] are edited with 25 optimization steps.

\paragraph{Memory-based Methods.}
\textbf{SERAC} pairs each base LLM with a smaller counterfactual model (Qwen2.5-0.5B-Instruct for Qwen; Llama-3.2-1B-Instruct for Llama3) and a DistilBERT scope classifier with a cache hit threshold of 0.5.
\textbf{GRACE} performs editing at layer 18 (Qwen) or layer 27 (Llama3), with edit lr 1.0, 50 iterations, $\epsilon = 1.0$, and Euclidean distance for key matching.
\textbf{MELO} builds upon GRACE by adding LoRA adaptors ($r = \alpha = 64$) on the two layers following the GRACE layer, with 350 blocks $\times$ 2 ranks per block (total rank 700), initial radius 0.5, and edit lr 0.001.

\section{Extended Experiments}
\label{app:extended_experiments}

\begin{table}[t]
  \centering
  \footnotesize
  \setlength{\tabcolsep}{2.5pt}
  \resizebox{\textwidth}{!}{%
  \begin{tabular}{@{}l *{8}{c}@{}}
  \toprule
  \multirow{2}{*}{\textbf{Method}}
    & \multicolumn{3}{c}{\textbf{Ori}} & \multicolumn{3}{c}{\textbf{Para}} & \multicolumn{2}{c}{\textbf{Sub}} \\
  \cmidrule(lr){2-4}\cmidrule(lr){5-7}\cmidrule(lr){8-9}
    & BERTScore & ROUGE-L & ROUGE-1 & BERTScore & ROUGE-L & ROUGE-1 & ROUGE-L & ROUGE-1 \\
  \midrule
  \multicolumn{9}{@{}l}{\textit{Llama3-8B-Instruct}} \\
  SERAC        & 70.15 & 27.89 & 29.73 & \uline{70.44} & \uline{26.91} & \uline{28.88} & 22.03 & 22.71 \\
  MELO         & 69.58 & 25.19 & 27.13 & 70.04 & 25.26 & 27.26 & 21.98 & 22.66 \\
  GRACE        & \uline{73.02} & \uline{34.34} & \uline{36.17} & 70.11 & 25.42 & 27.43 & \uline{22.41} & \uline{23.09} \\
  ManiEdit    & \textbf{76.22} & \textbf{37.30} & \textbf{39.66} & \textbf{75.80} & \textbf{35.61} & \textbf{38.11} & \textbf{32.26} & \textbf{33.28} \\
  \midrule
  \multicolumn{9}{@{}l}{\textit{Qwen2.5-7B-Instruct}} \\
  SERAC        & 71.73 & 30.93 & 33.14 & \uline{70.89} & \uline{29.15} & \uline{31.50} & 24.18 & 24.99 \\
  MELO         & 71.26 & 28.51 & 30.86 & 70.48 & 27.53 & 29.95 & 24.19 & 24.98 \\
  GRACE        & \uline{74.36} & \uline{37.12} & \uline{39.39} & 70.55 & 27.96 & 30.42 & \uline{24.58} & \uline{25.39} \\
  ManiEdit    & \textbf{76.63} & \textbf{39.43} & \textbf{42.38} & \textbf{74.97} & \textbf{38.07} & \textbf{40.97} & \textbf{32.03} & \textbf{33.05} \\
  \bottomrule
  \end{tabular}%
  }%
  \caption{Comparison of ManiEdit with memory-based methods under the sequential unstructured editing setting on UnKEBench.}
  \label{tab:app_unke_splits}
\end{table}

\begin{table}[t]
  \centering
  \footnotesize
  \setlength{\tabcolsep}{2.5pt}
  \resizebox{\textwidth}{!}{%
  \begin{tabular}{@{}l *{8}{c}@{}}
  \toprule
  \multirow{2}{*}{\textbf{Method}}
    & \multicolumn{3}{c}{\textbf{Ori}} & \multicolumn{3}{c}{\textbf{Para}} & \multicolumn{2}{c}{\textbf{Sub}} \\
  \cmidrule(lr){2-4}\cmidrule(lr){5-7}\cmidrule(lr){8-9}
    & BERTScore & ROUGE-L & ROUGE-1 & BERTScore & ROUGE-L & ROUGE-1 & ROUGE-L & ROUGE-1 \\
  \midrule
  \multicolumn{9}{@{}l}{\textit{Llama3-8B-Instruct}} \\
  SERAC        & 69.22 & \uline{23.20} & \uline{24.88} & 42.97 & 15.25 & 16.58 & \uline{37.36} & \uline{38.35} \\
  MELO         & \uline{69.69} & 20.30 & 22.22 & \uline{52.09} & \uline{19.11} & \uline{21.04} & 34.42 & 35.42 \\
  GRACE        & 69.05 & 19.60 & 21.25 & 41.90 & 12.97 & 14.18 & 33.92 & 34.98 \\
  ManiEdit    & \textbf{78.50} & \textbf{37.24} & \textbf{39.98} & \textbf{53.88} & \textbf{28.73} & \textbf{30.80} & \textbf{40.05} & \textbf{41.18} \\
  \midrule
  \multicolumn{9}{@{}l}{\textit{Qwen2.5-7B-Instruct}} \\
  SERAC        & 69.78 & \uline{21.76} & \uline{23.67} & \uline{52.00} & \uline{19.98} & \uline{21.96} & \uline{36.00} & \uline{36.93} \\
  MELO         & 68.69 & 18.59 & 20.24 & 41.91 & 12.98 & 14.19 & 33.59 & 34.68 \\
  GRACE        & \uline{69.84} & 20.63 & 22.53 & \textbf{52.07} & 19.12 & 21.03 & 34.61 & 35.59 \\
  ManiEdit    & \textbf{73.88} & \textbf{38.17} & \textbf{40.96} & 50.78 & \textbf{31.45} & \textbf{34.07} & \textbf{40.60} & \textbf{41.60} \\
  \bottomrule
  \end{tabular}%
  }%
  \caption{Comparison of ManiEdit with memory-based methods under the sequential unstructured editing setting on the AKEW (Counterfact) benchmark.}
  \label{tab:app_counterfact_splits}
\end{table}

\begin{table}[t]
  \centering
  \footnotesize
  \setlength{\tabcolsep}{2.5pt}
  \resizebox{\textwidth}{!}{%
  \begin{tabular}{@{}l *{8}{c}@{}}
  \toprule
  \multirow{2}{*}{\textbf{Method}}
    & \multicolumn{3}{c}{\textbf{Ori}} & \multicolumn{3}{c}{\textbf{Para}} & \multicolumn{2}{c}{\textbf{Sub}} \\
  \cmidrule(lr){2-4}\cmidrule(lr){5-7}\cmidrule(lr){8-9}
    & BERTScore & ROUGE-L & ROUGE-1 & BERTScore & ROUGE-L & ROUGE-1 & ROUGE-L & ROUGE-1 \\
  \midrule
  \multicolumn{9}{@{}l}{\textit{Llama3-8B-Instruct}} \\
  SERAC        & 69.97 & 19.94 & 21.50 & 69.75 & \uline{21.20} & \uline{23.54} & \textbf{42.36} & \textbf{43.64} \\
  MELO         & 69.52 & 17.00 & 18.54 & 69.78 & 21.14 & 23.48 & 39.95 & 41.35 \\
  GRACE        & \textbf{81.82} & \textbf{52.57} & \textbf{53.37} & \uline{69.82} & 21.13 & 23.45 & \uline{39.97} & \uline{41.38} \\
  ManiEdit    & \uline{71.51} & \uline{30.03} & \uline{32.57} & \textbf{71.24} & \textbf{35.34} & \textbf{38.35} & 32.70 & 34.05 \\
  \midrule
  \multicolumn{9}{@{}l}{\textit{Qwen2.5-7B-Instruct}} \\
  SERAC        & 67.67 & 22.72 & 24.84 & 68.35 & \uline{20.79} & \uline{23.14} & \uline{42.89} & \uline{44.00} \\
  MELO         & 68.27 & 19.62 & 21.75 & \uline{68.37} & 20.68 & 23.02 & 41.62 & 42.73 \\
  GRACE        & \textbf{80.56} & \textbf{48.33} & \textbf{49.75} & 68.35 & 20.66 & 23.02 & 41.72 & 42.80 \\
  ManiEdit    & \uline{71.37} & \uline{37.96} & \uline{41.15} & \textbf{70.87} & \textbf{35.08} & \textbf{38.25} & \textbf{44.74} & \textbf{45.88} \\
  \bottomrule
  \end{tabular}%
  }%
  \caption{Comparison of ManiEdit with memory-based methods under the sequential unstructured editing setting on the AKEW (MQUAKE) benchmark.}
  \label{tab:app_mquake_splits}
\end{table}

\subsection{Knowledge Editing Performance: ManiEdit vs. Memory-based Baselines}
\label{app:extended_splits}

We supplement Table~\ref{tab:main_results} with BERTScore, ROUGE-L, and ROUGE-1 (all in \%) under the Original (Ori), Paraphrase (Para), and Sub-question (Sub) splits on UnKEBench, AKEW (Counterfact), and AKEW (MQUAKE), comparing memory-based baselines---SERAC, MELO, and GRACE---against ManiEdit on both Llama3-8B-Instruct and Qwen2.5-7B-Instruct. For each backbone and each metric, the highest value is in bold and the second-highest is underlined; detailed results are provided in Tables~\ref{tab:app_unke_splits}, \ref{tab:app_counterfact_splits}, and \ref{tab:app_mquake_splits}.

The results reveal a consistent pattern that distinguishes memory-based retrieval from genuine knowledge internalization:

\begin{itemize}[leftmargin=*,nosep]
  \item \textbf{ManiEdit achieves the best Para scores across nearly all benchmarks, demonstrating genuine knowledge internalization.} Across BERTScore, ROUGE-L, and ROUGE-1 on the three benchmarks and both backbones (18 Para columns), ManiEdit ranks first in 17 columns. Relative to the second-best memory-based method, the largest leads are $+5.36$ BERTScore, $+14.29$ ROUGE-L, and $+15.11$ ROUGE-1, with average leads of $+2.31$, $+11.19$, and $+11.75$, respectively. These Para gains indicate that ManiEdit internalizes edited knowledge rather than relying on external routing like memory-based methods.
  \item \textbf{On Sub-question evaluation, ManiEdit remains competitive while memory-based methods show inconsistent fine-grained recall of the edits.} SERAC leads Sub scores on MQUAKE Llama3. However, ManiEdit consistently leads Sub scores on UnKEBench and AKEW (Counterfact), with average leads of $+6.15$ ROUGE-L and $+6.34$ ROUGE-1 over the second-best memory-based method, reflecting more robust fine-grained recall of the edited knowledge.
\end{itemize}

\subsection{Sub-Question Evaluation for Model Editing Methods}
\label{app:sub_parametric}

Section~\ref{sec:experiments} reports the \textbf{Original} and \textbf{Paraphrase} splits in Table~\ref{tab:main_results}. Table~\ref{tab:app_sub_parametric} further reports \textbf{Sub-question (Sub)} results for model editing methods, using ROUGE-1 and ROUGE-L (both in \%) on the same backbones and benchmarks. Sub questions probe fine-grained facets of the edited knowledge, providing a more targeted test of whether the update supports fine-grained factual recall beyond the main query. Each column's best is bold and the runner-up is underlined.

\begin{table}[t]
  \centering
  \footnotesize
  \setlength{\tabcolsep}{1.8pt}
  \resizebox{\textwidth}{!}{%
  \begin{tabular}{@{}l *{12}{c}@{}}
  \toprule
  \multirow{3}{*}{\textbf{Method}}
    & \multicolumn{6}{c}{\textbf{Qwen2.5-7B-Instruct}}
    & \multicolumn{6}{c}{\textbf{Llama3-8B-Instruct}} \\
  \cmidrule(lr){2-7}\cmidrule(lr){8-13}
    & \multicolumn{2}{c}{\textbf{UnKEBench}} & \multicolumn{2}{c}{\textbf{AKEW (MQUAKE)}} & \multicolumn{2}{c}{\textbf{AKEW (Counterfact)}}
    & \multicolumn{2}{c}{\textbf{UnKEBench}} & \multicolumn{2}{c}{\textbf{AKEW (MQUAKE)}} & \multicolumn{2}{c}{\textbf{AKEW (Counterfact)}} \\
  \cmidrule(lr){2-3}\cmidrule(lr){4-5}\cmidrule(lr){6-7}\cmidrule(lr){8-9}\cmidrule(lr){10-11}\cmidrule(lr){12-13}
    & ROUGE-1 & ROUGE-L & ROUGE-1 & ROUGE-L & ROUGE-1 & ROUGE-L
    & ROUGE-1 & ROUGE-L & ROUGE-1 & ROUGE-L & ROUGE-1 & ROUGE-L \\
  \midrule
  MEMIT          & 16.54 & 16.35 & 23.06 & 22.82 & 26.57 & 25.80
                 & 1.39  & 1.35  & 10.33 & 10.00 & 13.09 & 12.80 \\
  RECT           & 15.52 & 15.31 & 26.10 & 25.58 & 24.39 & 24.10
                 & 15.00 & 14.63 & 17.52 & 17.38 & 16.19 & 16.13 \\
  UnKE           & 6.72  & 6.38  & 25.03 & 26.62 & 11.26 & 11.09
                 & 0.69  & 0.68  & 5.51  & 5.32  & 4.06 & 4.03 \\
  AnyEdit        & 15.63 & 15.29 & 21.63 & 21.14 & 21.82 & 21.37
                 & 15.84 & 15.62 & 21.01 & 20.43 & 13.22 & 12.98 \\
  EvoEdit        & 14.72 & 14.34 & 22.69 & 21.98 & 26.72 & 26.11
                 & 11.98 & 11.79 & 16.20 & 15.71 & 18.22 & 17.84 \\
  MOSE           & \uline{28.63} & \uline{27.39} & \uline{37.81} & \uline{37.16} & \uline{31.13} & \uline{30.28}
                 & \uline{27.58} & \uline{26.88} & \textbf{45.09} & \textbf{44.41} & \uline{35.75} & \uline{35.00} \\
  \cmidrule(lr){1-13}
  ManiEdit       & \textbf{33.05} & \textbf{32.03} & \textbf{45.88} & \textbf{44.74} & \textbf{41.60} & \textbf{40.60}
                 & \textbf{33.28} & \textbf{32.26} & \uline{34.05} & \uline{32.70} & \textbf{41.18} & \textbf{40.05} \\
  \bottomrule
  \end{tabular}%
  }
  \caption{Comparison of ManiEdit with existing methods on the sequential unstructured model editing tasks under the Sub-question split. The best results are highlighted in bold, while the second-best results are underlined.}
  \label{tab:app_sub_parametric}
\end{table}

\begin{itemize}[leftmargin=*,nosep]
\item \textbf{ManiEdit provides the most reliable fine-grained factual recall among all editors.} ManiEdit ranks first in 10 of 12 Sub columns. Across these 10 columns, ManiEdit outperforms the strongest baseline by $+6.82$ ROUGE-1 and $+6.59$ ROUGE-L on average. Overall, the Sub split confirms that ManiEdit improves not only direct and paraphrased recall, but also fine-grained recall of edited knowledge.
\end{itemize}

\subsection{General Capability Preservation}
\label{app:extended_glue}

To complement the three GLUE benchmarks reported in Section~\ref{sec:experiments} (SST-2, CoLA, MRPC), we conduct the same 500-edit sequential stability experiment on three additional benchmarks: \textbf{MMLU}~\cite{hendrycks2021mmlu}, \textbf{MNLI}~\cite{williams2018mnli}, and \textbf{RTE}~\cite{bentivogli2009rte}. All experiments use Qwen2.5-7B-Instruct and follow the same evaluation protocol as Section~\ref{sec:experiments}: accuracy is measured at intervals of 100 cumulative edits against MOSE, AnyEdit, and RECT.

\begin{figure}[t]
  \centering
  \includegraphics[width=\textwidth]{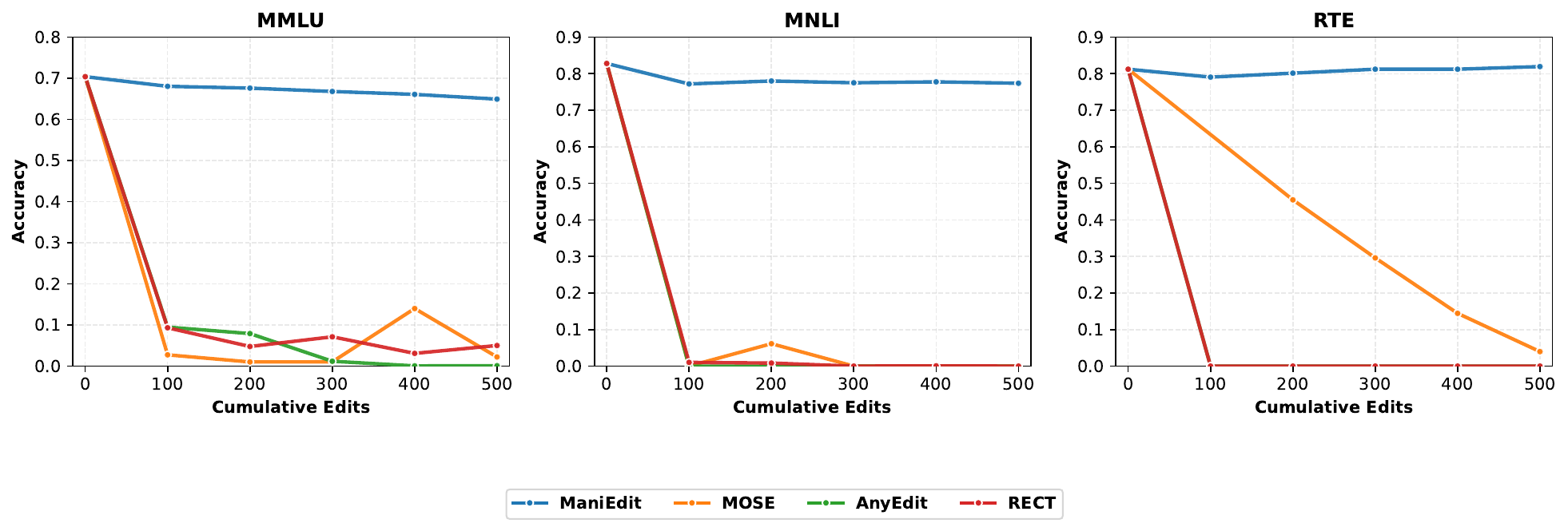}
  \caption{General capability accuracy on MMLU, MNLI, and RTE as the number of sequential unstructured edits increases. ManiEdit maintains near-original accuracy across all tasks, while baselines collapse rapidly.}
  \label{fig:glue_appendix}
\end{figure}

Figure~\ref{fig:glue_appendix} shows the accuracy curves. The findings are consistent with the main-paper results:

\begin{itemize}[leftmargin=*,nosep]
  \item \textbf{AnyEdit and RECT collapse within 100 edits} on all three benchmarks. On MNLI and RTE, both methods drop to 0.0\% accuracy after only 100 edits and remain at zero throughout. On MMLU, RECT falls from 70.4\% to 9.3\%, and AnyEdit reaches 9.5\% by edit 100.
  \item \textbf{MOSE collapses abruptly and erratically.} On MMLU and MNLI it drops to near-zero within 100 edits, with highly unstable fluctuations across subsequent checkpoints rather than any consistent recovery. On RTE, accuracy declines steadily from the pre-edit 81.2\% to 45.5\% at 200 edits, and further to 4.0\% at 500 edits.
  \item \textbf{ManiEdit maintains stable performance} throughout the full 500-edit sequence. On MMLU it declines by only $\approx$5.5\% (70.4\%$\to$64.9\%); on MNLI it stays within 5.4\% of the original (82.8\%$\to$77.4\%); on RTE it stays near the pre-edit accuracy and finishes slightly higher (81.2\%$\to$81.9\%). This broad consistency across diverse benchmarks confirms that ManiEdit's differential constraint design provides robust protection of pre-trained capabilities under extended sequential editing.
\end{itemize}

\begin{figure}[t]
  \centering
  \includegraphics[width=\textwidth]{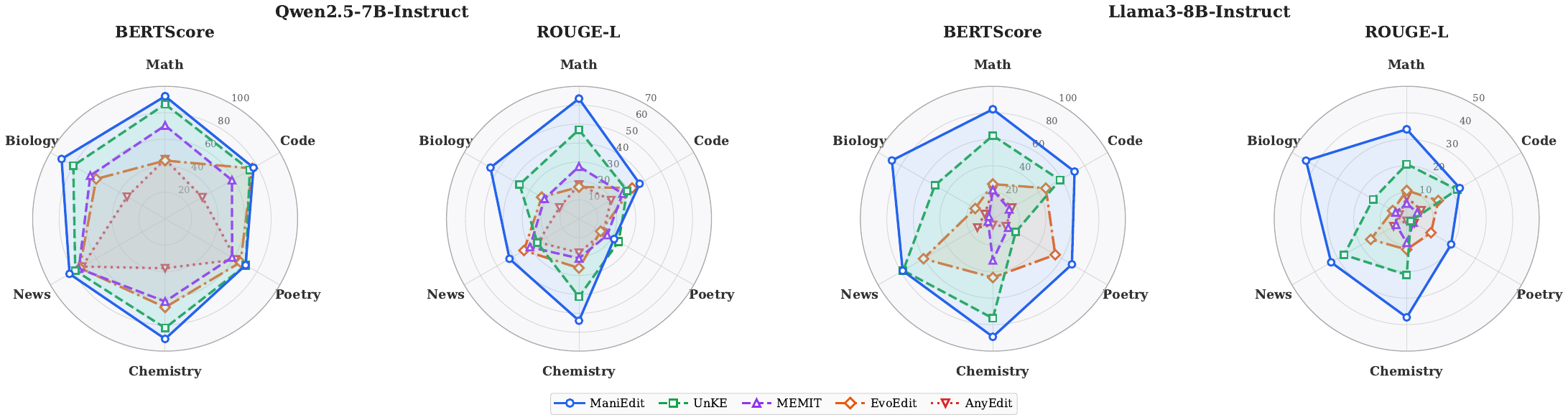}
  \caption{Radar plots of per-domain editing performance on EditEverything. From left to right: BERTScore and ROUGE-L (\%) on Qwen2.5-7B-Instruct, then BERTScore and ROUGE-L (\%) on Llama3-8B-Instruct. Each axis is one domain; a larger enclosed area indicates better overall performance.}
  \label{fig:editeverything_radar}
\end{figure}

\subsection{Cross-Domain Editing Performance on EditEverything}
\label{app:editeverything}

Figure~\ref{fig:editeverything_radar} presents per-domain BERTScore and ROUGE-L on the EditEverything benchmark~\cite{jiang2025anyedit} across six specialized knowledge categories---Math, Code, Poetry, Chemistry, News, and Biology---on both Qwen2.5-7B-Instruct and Llama3-8B-Instruct.

The radar plots reveal several consistent cross-domain patterns:

\begin{itemize}[leftmargin=*,nosep]
  \item \textbf{ManiEdit spans the largest radar area across all six domains on both models.}
  On Qwen2.5-7B-Instruct, ManiEdit achieves the highest BERTScore and ROUGE-L in five of six domains: Math, Chemistry, Biology, News, and Code, with the sole exception of Poetry where UnKE achieves a marginally higher score.
  On Llama3-8B-Instruct, ManiEdit leads all six domains unambiguously. Averaged across domains, ManiEdit outperforms the second-best method by approximately 16 BERTScore points on Llama3 and 5 BERTScore points on Qwen2.5.
  \item \textbf{Scientific domains yield high editing accuracy; open-ended domains remain challenging.}
  Across all methods and both backbones, Math, Chemistry, and Biology consistently attain higher scores than Poetry, Code, and News. The regularity and self-consistency of mathematical and scientific knowledge facilitate accurate weight-level encoding. Poetry records the lowest scores for all methods---reflecting the stylistic variability of open-ended text and the difficulty of producing surface-faithful completions. However, ManiEdit still maintains 70.13 (Qwen2.5) and 68.91 (Llama3) BERTScore even in this hardest category.
\end{itemize}

\subsection{Analysis of Pivot Localization under the Original-Query Setting}
\label{app:ablation_ori}

Figure~\ref{fig:ablation_ori} complements the Para-query analysis in Section~\ref{sec:ablation} by reporting Ori ROUGE-L under the same six backbone--benchmark settings. This view tests whether Pivot Localization also improves direct editing efficacy when the query exactly matches the edited knowledge.

\begin{figure}[t]
  \centering
  \includegraphics[width=\textwidth]{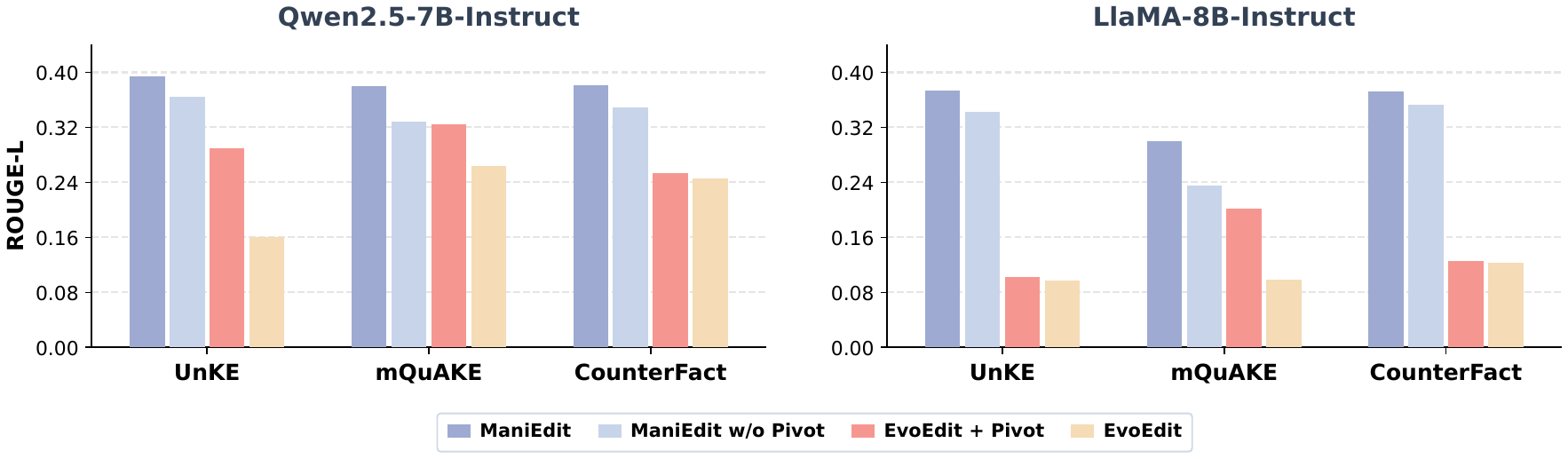}
  \caption{Analysis of Pivot Localization under the Original-Query setting. Scores are Ori ROUGE-L on two base LLMs across three benchmarks. ManiEdit w/o Pivot and EvoEdit use the AnyEdit chunking protocol.}
  \label{fig:ablation_ori}
\end{figure} 

The Ori-query results reinforce the plug-in value of Pivot Localization for locate-then-edit methods. Adding Pivot Localization to EvoEdit improves ROUGE-L in all six settings, with especially large gains on Qwen-UnKEBench ($16.08 \to 28.93$), Qwen-MQUAKE ($26.35 \to 32.50$), and Llama-MQUAKE ($9.89 \to 20.23$). This suggests that locating high-leverage semantic pivots can directly improve the quality of the update target.

For ManiEdit, removing Pivot Localization lowers ROUGE-L in every setting; the full model improves over the no-pivot variant by $+3.83$ ROUGE-L on average, ranging from $+1.92$ on Llama-Counterfact to $+6.54$ on Llama-MQUAKE. 
Together, the Ori and Para analyses show that Pivot Localization improves both paraphrase robustness and original-query effectiveness in every setting.

\begin{table}[t]
  \centering
  \footnotesize
  \setlength{\tabcolsep}{2.0pt}
  \resizebox{\textwidth}{!}{%
  \begin{tabular}{@{}llllcccc@{}}
  \toprule
  \multirow{2}{*}{\textbf{Backbone}} & \multirow{2}{*}{\textbf{Benchmark}}
  & \multirow{2}{*}{\textbf{General-Capability Preservation}} & \multirow{2}{*}{\textbf{Edit-Knowledge Preservation}}
  & \multicolumn{2}{c}{\textbf{Original}} & \multicolumn{2}{c}{\textbf{Para}} \\
  \cmidrule(lr){5-6}\cmidrule(lr){7-8}
  & & & & BERTScore & ROUGE-L & BERTScore & ROUGE-L \\
  \midrule
  \multirow{12}{*}{Qwen2.5-7B-Instruct}
  & \multirow{4}{*}{UnKEBench}
  & Null-space Projection & Recursive Null-space Alignment & 43.69 & 24.12 & 41.85 & 23.64 \\
  & & Energy-weighted Penalty & Regularization & 46.73 & 29.14 & 46.20 & 28.46 \\
  & & Null-space Projection & Regularization & 37.08 & 30.90 & 37.99 & 31.45 \\
  & & Energy-weighted Penalty & Recursive Null-space Alignment & \textbf{76.63} & \textbf{39.43} & \textbf{74.97} & \textbf{38.07} \\
  \cmidrule(lr){2-8}
  & \multirow{4}{*}{AKEW (Counterfact)}
  & Null-space Projection & Recursive Null-space Alignment & 45.76 & 27.54 & 30.49 & 24.60 \\
  & & Energy-weighted Penalty & Regularization & 62.26 & 32.93 & 50.33 & 30.52 \\
  & & Null-space Projection & Regularization & 48.74 & 31.25 & 44.40 & 30.03 \\
  & & Energy-weighted Penalty & Recursive Null-space Alignment & \textbf{73.88} & \textbf{38.17} & \textbf{50.78} & \textbf{31.45} \\
  \cmidrule(lr){2-8}
  & \multirow{4}{*}{AKEW (MQUAKE)}
  & Null-space Projection & Recursive Null-space Alignment & 64.68 & 36.56 & 60.66 & 30.54 \\
  & & Energy-weighted Penalty & Regularization & 64.57 & 33.84 & 62.77 & 30.72 \\
  & & Null-space Projection & Regularization & 36.97 & 32.53 & 38.90 & 31.50 \\
  & & Energy-weighted Penalty & Recursive Null-space Alignment & \textbf{71.37} & \textbf{37.96} & \textbf{70.87} & \textbf{35.08} \\
  \midrule
  \multirow{12}{*}{Llama3-8B-Instruct}
  & \multirow{4}{*}{UnKEBench}
  & Null-space Projection & Recursive Null-space Alignment & 8.17 & 9.13 & 7.18 & 9.24 \\
  & & Energy-weighted Penalty & Regularization & 3.95 & 6.40 & 4.19 & 6.83 \\
  & & Null-space Projection & Regularization & 0.00 & 7.72 & 0.37 & 8.33 \\
  & & Energy-weighted Penalty & Recursive Null-space Alignment & \textbf{76.22} & \textbf{37.30} & \textbf{75.80} & \textbf{35.61} \\
  \cmidrule(lr){2-8}
  & \multirow{4}{*}{AKEW (Counterfact)}
  & Null-space Projection & Recursive Null-space Alignment & 13.63 & 11.14 & 5.78 & 9.74 \\
  & & Energy-weighted Penalty & Regularization & 6.04 & 3.81 & 4.10 & 4.74 \\
  & & Null-space Projection & Regularization & 9.07 & 7.12 & 6.65 & 8.07 \\
  & & Energy-weighted Penalty & Recursive Null-space Alignment & \textbf{78.50} & \textbf{37.24} & \textbf{53.88} & \textbf{28.73} \\
  \cmidrule(lr){2-8}
  & \multirow{4}{*}{AKEW (MQUAKE)}
  & Null-space Projection & Recursive Null-space Alignment & 52.25 & 19.60 & 49.53 & 19.54 \\
  & & Energy-weighted Penalty & Regularization & 3.40 & 3.09 & 1.93 & 1.99 \\
  & & Null-space Projection & Regularization & 5.64 & 7.32 & 7.52 & 10.53 \\
  & & Energy-weighted Penalty & Recursive Null-space Alignment & \textbf{71.51} & \textbf{30.03} & \textbf{71.24} & \textbf{35.34} \\
  \bottomrule
  \end{tabular}%
  }
  \caption{Ablation of Manifold-Aware Preservation components on UnKEBench, AKEW (Counterfact), and AKEW (MQUAKE). The left component in each pair protects general pre-trained capabilities, and the right component protects sequentially edited knowledge.}
  \label{tab:app_matp_component_ablation}
\end{table}

\subsection{Ablation of Manifold-Aware Preservation Components}
\label{app:matp_component_ablation}

To analyze the two protection mechanisms in Manifold-Aware Preservation, Table~\ref{tab:app_matp_component_ablation} compares four combinations of constraints. The left component protects general pre-trained capabilities and the right component preserves sequentially edited knowledge; ManiEdit is the \textit{Energy-weighted Penalty + Recursive Null-space Alignment} combination. 

We instantiate the four combinations as follows. Let $K_0$ collect keys of preserved pre-trained knowledge. Following AlphaEdit~\cite{fang2025alphaedit}, we take $P^{\mathrm{NS}}=\hat{U}\hat{U}^{\top}$, where $\hat{U}$ retains the eigenvectors of $K_0K_0^{\top}$ associated with zero eigenvalues (in practice, those below a small threshold), so that $P^{\mathrm{NS}}$ maps the update into the null space of $K_0$. Let $K_{<t}$ collect all keys of edits $i<t$, and let $P_{t-1}$ be the recursive null-space alignment of Equation~\ref{eq:recursive_null_space_alignment}. 

\paragraph{Null-space Projection + Recursive Null-space Alignment.}
Initialize $P_0=P^{\mathrm{NS}}$ and recursively update $P_{t-1}$ by Equation~\ref{eq:recursive_null_space_alignment}, so both pre-trained and historical keys are protected by the null space:
\begin{equation}
\min_{\Delta_t}\ \bigl\|(W_{t-1}+\Delta_t P_{t-1})K_t-V_t^*\bigr\|_F^2+\lambda\|\Delta_t\|_F^2.
\end{equation}

\paragraph{Energy-weighted Penalty + Regularization.}
The energy-weighted penalty $\lambda\|\Delta_t\Sigma^{1/2}\|_F^2$ preserves pre-trained geometry, and $\|\Delta_t K_{<t}\|_F^2$ limits the output displacement of historical keys:
\begin{equation}
\min_{\Delta_t}\ \bigl\|(W_{t-1}+\Delta_t)K_t-V_t^*\bigr\|_F^2+\lambda\|\Delta_t\Sigma^{1/2}\|_F^2+\beta\|\Delta_t K_{<t}\|_F^2.
\end{equation}

\paragraph{Null-space Projection + Regularization.}
$P^{\mathrm{NS}}$ preserves pre-trained keys, while historical edits are protected by a soft penalty:
\begin{equation}
\min_{\Delta_t}\ \bigl\|(W_{t-1}+\Delta_t P^{\mathrm{NS}})K_t-V_t^*\bigr\|_F^2+\beta\|\Delta_t K_{<t}\|_F^2.
\end{equation}

\paragraph{Energy-weighted Penalty + Recursive Null-space Alignment.}
Initialize $P_0=I$ and recursively update $P_{t-1}$ by Equation~\ref{eq:recursive_null_space_alignment}. The energy-weighted penalty preserves pre-trained geometry, and $P_{t-1}$ protects historical edits:
\begin{equation}
\min_{\Delta_t}\ \bigl\|(W_{t-1}+\Delta_t P_{t-1})K_t-V_t^*\bigr\|_F^2+\lambda\|\Delta_t\Sigma^{1/2}\|_F^2.
\end{equation}

The results show that neither purely soft regularization nor a uniform projection reliably protects both knowledge types. Replacing either side of ManiEdit's asymmetric design substantially degrades editing quality in all settings, especially on Llama3-8B-Instruct. The full combination of energy-weighted preservation for the pre-trained manifold and recursive null-space alignment for historical edits achieves optimal performance.

\subsection{Long-Horizon Sequential Editing Stress Test}
\label{app:stress_test}
To further assess robustness under longer edit sequences, we extend the editing horizon from 500 edits to 1{,}000 edits on UnKEBench. Table~\ref{tab:stress_llama_unke} reports a comparison against MOSE on Llama3-8B-Instruct with UnKEBench; the percentage-point change $\Delta$ quantifies how much editing performance drifts as the sequence length doubles.

\begin{table}[t]
  \centering
  \footnotesize
  \setlength{\tabcolsep}{2.5pt}
  \resizebox{\linewidth}{!}{%
  \begin{tabular}{@{}l *{8}{c}@{}}
  \toprule
  \multirow{3}{*}{\textbf{Setting}}
    & \multicolumn{4}{c}{\textbf{MOSE}}
    & \multicolumn{4}{c}{\textbf{ManiEdit}} \\
  \cmidrule(lr){2-5}\cmidrule(lr){6-9}
    & \multicolumn{2}{c}{\textbf{Ori}} & \multicolumn{2}{c}{\textbf{Para}}
    & \multicolumn{2}{c}{\textbf{Ori}} & \multicolumn{2}{c}{\textbf{Para}} \\
  \cmidrule(lr){2-3}\cmidrule(lr){4-5}\cmidrule(lr){6-7}\cmidrule(lr){8-9}
    & BERTScore & ROUGE-L & BERTScore & ROUGE-L
    & BERTScore & ROUGE-L & BERTScore & ROUGE-L \\
  \midrule
  @500      & 62.38 & 31.76 & 61.11 & 30.95 & 76.22 & 37.30 & 75.80 & 35.61 \\
  @1{,}000  & 54.65 & 20.57 & 55.56 & 19.80 & 77.00 & 34.99 & 75.66 & 34.61 \\
  $\Delta$  & $-7.73$ & $-11.19$ & $-5.55$ & $-11.15$ & $+0.78$ & $-2.31$ & $-0.14$ & $-1.00$ \\
  \bottomrule
  \end{tabular}%
  }
  \caption{Long-horizon stress test on Llama3-8B-Instruct with UnKEBench (scores in \%). $\Delta$ denotes the change from 500 to 1{,}000 sequential edits.}
  \label{tab:stress_llama_unke}
\end{table}

\paragraph{ManiEdit maintains its effectiveness under long-horizon editing.} After doubling the edit horizon, ManiEdit's Para and Ori ROUGE-L drop by merely 1.00 and 2.31, whereas MOSE declines sharply by over 11 ROUGE-L points on both Para and Ori. This demonstrates ManiEdit's strong robustness under long-horizon unstructured editing.

\subsection{Preservation of Historical Edits via Recursive Null-space Alignment}
\label{app:leakage}
The closed-form solution (Equation~\ref{eq:closed_form}) can be applied in two ways: $W_t = W_{t-1} + \Delta_t$ or $W_t = W_{t-1} + \Delta_t P_{t-1}$. This paper adopts the former, so that $P_{t-1}$ in the optimization objective (Equation~\ref{eq:objective}) guides $\Delta_t$ toward the orthogonal complement of $\mathrm{span}\{K_i\}_{i<t}$, driving $\Delta_t K_i$ toward zero. This section empirically shows that this constraint substantially reduces the interference of a new update with historical edit keys, thereby preserving historical edits.

Writing $\Delta_t = R_tF_t$ with $F_t := (\lambda I + G_t)^{-1}A_t^\top\Sigma^{-1}$, $A_t := P_{t-1}K_t$, and $G_t := A_t^\top\Sigma^{-1}A_t$ (Appendix~\ref{app:soft_proof}), the residual $R_t$ enters every response linearly, so we evaluate the update through $F_t$. Since $F_t$ depends on $P_{t-1}$, we compute $P_{t-1}$ with $\epsilon = 0$ in line with Proposition~\ref{prop:hard}, and compare against a baseline that replaces $P_{t-1}$ by $I$ (i.e., the same energy-weighted update without recursive null-space alignment). We report

\begin{equation*}
\mathrm{leak}_t = \frac{\|F_tK_{<t}\|_F}{\|F_tK_t\|_F},
\qquad 
\mathrm{cum}_t = \frac{\bigl(\sum_{s>t}\|F_sK_t\|_F^2\bigr)^{1/2}}{\|F_tK_t\|_F},
\end{equation*}
where $K_{<t}$ collects all keys of edits $i < t$.
Thus $\mathrm{leak}_t$ measures how strongly the $t$-th update applies its residual to all previous edit keys, relative to its own keys; $\mathrm{cum}_t$ measures how strongly later updates apply their residuals to edit $t$, relative to that edit's own update.
This measurement uses the second-moment matrix $\Sigma \in \mathbb{R}^{n \times n}$ ($n = 18{,}944$) in the key space of the layer-8 FFN down-projection matrix in Qwen2.5-7B-Instruct, together with the edit keys of the 500 UnKEBench samples (1{,}427 keys in total).

\begin{table}[t]
  \centering 
  \footnotesize
  \setlength{\tabcolsep}{4pt}
  \begin{tabular}{@{}lcccc@{}}
  \toprule
  \multirow{2}{*}{\textbf{Update}}
    & \multicolumn{3}{c}{$\mathrm{leak}_t \downarrow$}
    & \multicolumn{1}{c}{$\mathrm{cum}_t \downarrow$} \\
  \cmidrule(lr){2-4}\cmidrule(lr){5-5}
    & Median & Mean & Max & Median  \\
  \midrule
  Without alignment ($P_{t-1} = I$) & 9.02 & 8.62 & 16.0 & 8.72 \\
  ManiEdit                          & 0.305 & 0.343 & 2.09 & 0.297 \\
  \midrule
  Reduction (median)                & \multicolumn{3}{c}{$29.6\times$} & $29.4\times$ \\
  \bottomrule
  \end{tabular}
  \caption{$\mathrm{leak}_t$ and $\mathrm{cum}_t$ over 500 sequential edits, with and without recursive null-space alignment.}
  \label{tab:leakage} 
\end{table}

Table~\ref{tab:leakage} shows that without alignment, the median $\mathrm{leak}_t$ is $9.02$ and the median $\mathrm{cum}_t$ is $8.72$. With alignment, these drop to $0.305$ and $0.297$, a reduction of about $30\times$ in both.
These results show that recursive null-space alignment effectively guides $\Delta_t$ toward the orthogonal complement of $\mathrm{span}\{K_i\}_{i<t}$, thereby preserving historical edits against both a new update and those that follow.

\section{Derivation of the Closed-Form Solution}
\label{app:derivation}

In this section, we provide the rigorous mathematical derivation for the closed-form solution of our unified optimization objective.

\subsection{Problem Formulation}
At editing step $t$, we seek to find the optimal weight increment $\Delta_t \in \mathbb{R}^{m \times n}$ that minimizes the following objective function:
\begin{equation}
\begin{split}
\min_{\Delta_t} \mathcal{J}(\Delta_t) := &\left\| (W_{t-1} + \Delta_t P_{t-1}) K_t - V_t^* \right\|_F^2 \\
&+ \lambda \left\| \Delta_t \Sigma^{1/2} \right\|_F^2,
\end{split}
\end{equation}
where:
\begin{itemize}
    \item $K_t \in \mathbb{R}^{n \times b}$ is the current input key matrix (representing a batch of $b$ keys).
    \item $V_t^* \in \mathbb{R}^{m \times b}$ is the target output matrix.
    \item $P_{t-1} \in \mathbb{R}^{n \times n}$ is a symmetric matrix ($P_{t-1}^\top = P_{t-1}$).
    \item $\Sigma \in \mathbb{R}^{n \times n}$ is the symmetric positive-definite feature second-moment matrix.
    \item $\lambda > 0$ is the regularization coefficient.
\end{itemize}
Let $R_t := V_t^* - W_{t-1} K_t \in \mathbb{R}^{m \times b}$ denote the current residual matrix. The objective simplifies to:
\begin{equation}
\mathcal{J}(\Delta_t) = \left\| \Delta_t P_{t-1} K_t - R_t \right\|_F^2 + \lambda \left\| \Delta_t \Sigma^{1/2} \right\|_F^2.
\end{equation}

\subsection{Matrix Trace Reformulation}
To facilitate differentiation, we rewrite the objective using matrix trace properties: $\|A\|_F^2 = \operatorname{Tr}(AA^\top)$ for a matrix $A$.
Expanding the mapping loss term:
\begin{equation}
\begin{aligned}
&\left\| \Delta_t P_{t-1} K_t - R_t \right\|_F^2 \\
&= \operatorname{Tr}\left( (\Delta_t P_{t-1} K_t - R_t)(\Delta_t P_{t-1} K_t - R_t)^\top \right) \\
&= \operatorname{Tr}\left( \Delta_t P_{t-1} K_t K_t^\top P_{t-1}^\top \Delta_t^\top - R_t K_t^\top P_{t-1}^\top \Delta_t^\top \right. \\
&\quad \left. - \Delta_t P_{t-1} K_t R_t^\top + R_t R_t^\top \right).
\end{aligned}
\end{equation}
Since $\operatorname{Tr}(A) = \operatorname{Tr}(A^\top)$ and $P_{t-1}$ is symmetric, we have $\operatorname{Tr}(R_t K_t^\top P_{t-1} \Delta_t^\top) = \operatorname{Tr}(\Delta_t P_{t-1} K_t R_t^\top)$. The mapping loss becomes:
\begin{equation}
\begin{aligned}
&\operatorname{Tr}\left( \Delta_t (P_{t-1} K_t K_t^\top P_{t-1}) \Delta_t^\top \right) \\
&\quad - 2\operatorname{Tr}\left( R_t K_t^\top P_{t-1} \Delta_t^\top \right) \\
&\quad + \operatorname{Tr}(R_t R_t^\top).
\end{aligned}
\end{equation}
Expanding the regularization term:
\begin{equation}
\begin{split}
\lambda \left\| \Delta_t \Sigma^{1/2} \right\|_F^2 &= \lambda \operatorname{Tr}\left( \Delta_t \Sigma^{1/2} (\Sigma^{1/2})^\top \Delta_t^\top \right) \\
&= \lambda \operatorname{Tr}\left( \Delta_t \Sigma \Delta_t^\top \right).
\end{split}
\end{equation}
Ignoring the constant term $\operatorname{Tr}(R_t R_t^\top)$ which does not depend on $\Delta_t$, the objective function becomes:
\begin{equation}
\begin{aligned}
\mathcal{J}(\Delta_t) &= \operatorname{Tr}\left( \Delta_t (P_{t-1} K_t K_t^\top P_{t-1}) \Delta_t^\top \right) \\
&\quad - 2\operatorname{Tr}\left( R_t K_t^\top P_{t-1} \Delta_t^\top \right) \\
&\quad + \lambda \operatorname{Tr}\left( \Delta_t \Sigma \Delta_t^\top \right).
\end{aligned}
\end{equation}

\subsection{Optimization and Closed-Form Solution}
We compute the gradient of $\mathcal{J}(\Delta_t)$ with respect to $\Delta_t$. Using the matrix calculus identities $\frac{\partial \operatorname{Tr}(X C X^\top)}{\partial X} = X(C + C^\top)$ and $\frac{\partial \operatorname{Tr}(D X^\top)}{\partial X} = D$, and noting that both $(P_{t-1} K_t K_t^\top P_{t-1})$ and $\Sigma$ are symmetric:
\begin{equation}
\begin{aligned}
\frac{\partial \mathcal{J}(\Delta_t)}{\partial \Delta_t} &= 2 \Delta_t (P_{t-1} K_t K_t^\top P_{t-1}) \\
&\quad - 2 R_t K_t^\top P_{t-1} + 2 \lambda \Delta_t \Sigma.
\end{aligned}
\end{equation}
Setting the gradient to zero to find the optimal $\Delta_t$:
\begin{equation}
\Delta_t (P_{t-1} K_t K_t^\top P_{t-1}) + \lambda \Delta_t \Sigma = R_t K_t^\top P_{t-1}.
\end{equation}
Factoring out $\Delta_t$ on the left yields:
\begin{equation}
\Delta_t \left( P_{t-1} K_t K_t^\top P_{t-1} + \lambda \Sigma \right) = R_t K_t^\top P_{t-1}.
\end{equation}
Let $M_t := P_{t-1} K_t K_t^\top P_{t-1} + \lambda \Sigma$. Since $P_{t-1} K_t K_t^\top P_{t-1}$ is positive semi-definite and $\lambda \Sigma$ is strictly positive definite, their sum $M_t$ is strictly positive definite and thus invertible. Right-multiplying by $M_t^{-1}$ gives the final closed-form solution:
\begin{equation}
\Delta_t = R_t K_t^\top P_{t-1} \left( P_{t-1} K_t K_t^\top P_{t-1} + \lambda \Sigma \right)^{-1}.
\end{equation}

\section{Proofs of Theoretical Propositions}
\label{app:spectral_derivation}

In this section we provide rigorous proofs of the propositions stated in Section~\ref{sec:theoretical_advantages}, thereby establishing the theoretical advantages of our method.

\subsection{Proof of Proposition~\ref{prop:leverage}}
\label{app:leverage_proof}
We consider the constrained problem
\begin{equation}
\label{eq:leverage_minimization}
\min_{\Delta \in \mathbb{R}^{m \times n}} \; \|\Delta\, \Sigma^{1/2}\|_F^2
\quad \text{s.t.} \quad \Delta\, k_t = R.
\end{equation}
Since
$\|\Delta \Sigma^{1/2}\|_F^2 = \operatorname{Tr}(\Delta \Sigma \Delta^\top) = \mathbb{E}_{k}\bigl[\|\Delta k\|^2\bigr]$,
where $k$ follows the distribution of pre-trained keys from which $\Sigma$ is estimated,
Equation~\ref{eq:leverage_minimization} measures the minimum perturbation of pre-trained outputs required to realize the output displacement $R$ at key $k_t$.

Introducing a vector Lagrange multiplier $\mu \in \mathbb{R}^{m}$,
\begin{equation}
\mathcal{L}(\Delta, \mu) = \|\Delta\, \Sigma^{1/2}\|_F^2 - 2\, \mu^\top \bigl(\Delta\, k_t - R\bigr).
\end{equation}
Setting $\partial \mathcal{L} / \partial \Delta = 2\Delta\, \Sigma - 2\,\mu\, k_t^\top = 0$ gives
\begin{equation}
\Delta^\star = \mu\, k_t^\top \Sigma^{-1}.
\end{equation}
Substituting into the constraint $\Delta^\star k_t = R$,
\begin{equation}
\mu\, (k_t^\top \Sigma^{-1} k_t) = R \quad \Longrightarrow \quad \mu = \frac{R}{k_t^\top \Sigma^{-1} k_t}.
\end{equation}
The optimal update is therefore
\begin{equation}
\Delta^\star = \frac{R\, k_t^\top \Sigma^{-1}}{k_t^\top \Sigma^{-1} k_t}.
\end{equation}
Plugging this back into the objective and using $\operatorname{Tr}(R^\top R) = \|R\|^2$ together with $k_t^\top \Sigma^{-1} \Sigma\, \Sigma^{-1} k_t = k_t^\top \Sigma^{-1} k_t$,
\begin{equation}
\begin{aligned}
\|\Delta^\star \Sigma^{1/2}\|_F^2
&= \operatorname{Tr}\!\bigl( \Sigma^{1/2} (\Delta^\star)^\top \Delta^\star \Sigma^{1/2} \bigr) \\
&= \frac{\operatorname{Tr}\!\bigl( R^\top R \bigr)\, k_t^\top \Sigma^{-1} k_t}{(k_t^\top \Sigma^{-1} k_t)^2} \\
&= \frac{\|R\|^2}{k_t^\top \Sigma^{-1} k_t}.
\end{aligned}
\end{equation}
The function $L \mapsto \|R\|^2 / L$ is strictly decreasing on $L > 0$, so the minimum energy-weighted cost is a strictly monotonically decreasing function of the leverage score $k_t^\top \Sigma^{-1} k_t$. This rigorously formalizes the intuition that high-leverage semantic pivots realize the target displacement with lower unintended perturbation, enabling precise editing at low cost.

\subsection{Proof of Proposition~\ref{prop:soft}}
\label{app:soft_proof}

We prove Proposition~\ref{prop:soft} in two steps. We first reduce the closed-form solution (Equation~\ref{eq:closed_form}) to a single-key spectral form via the Sherman--Morrison identity, then project the result into the eigenbasis of $\Sigma$ to read off the per-direction update allocation $\Delta W_{\text{dir}\ j} \propto \tilde{k}_j / \lambda_j$.

\paragraph{Single-Key Reduction via Sherman--Morrison.}
We begin from the closed-form solution derived in Appendix~\ref{app:derivation}:
\begin{equation}
\Delta_t = R_t K_t^\top P_{t-1} \bigl( P_{t-1} K_t K_t^\top P_{t-1} + \lambda \Sigma \bigr)^{-1}.
\end{equation}
Consider the single-key setting where $K_t$ collapses to a column vector $k_t \in \mathbb{R}^n$. Using the symmetry $P_{t-1}^\top = P_{t-1}$ and defining the projected key $\hat{k}_t := P_{t-1} k_t$, we have
\begin{equation}
K_t^\top P_{t-1} = \hat{k}_t^\top, \quad P_{t-1} K_t K_t^\top P_{t-1} = \hat{k}_t \hat{k}_t^\top,
\end{equation}
so the closed form reduces to
\begin{equation}
\label{eq:closed_form_single}
\Delta_t = R_t \hat{k}_t^\top \bigl( \hat{k}_t \hat{k}_t^\top + \lambda \Sigma \bigr)^{-1}.
\end{equation}
The inverse in Equation~\ref{eq:closed_form_single} can be evaluated analytically via the \textbf{Sherman--Morrison identity}: for any invertible $A$ and vectors $u, v$,
\begin{equation}
(A + u v^\top)^{-1} = A^{-1} - \frac{A^{-1} u v^\top A^{-1}}{1 + v^\top A^{-1} u}.
\end{equation}
Setting $A = \lambda \Sigma$ and $u = v = \hat{k}_t$, and noting $(\lambda \Sigma)^{-1} = \lambda^{-1} \Sigma^{-1}$,
\begin{equation}
\begin{aligned}
&\bigl( \lambda \Sigma + \hat{k}_t \hat{k}_t^\top \bigr)^{-1} \\
&= \tfrac{1}{\lambda} \Sigma^{-1} - \frac{\tfrac{1}{\lambda^2} \Sigma^{-1} \hat{k}_t \hat{k}_t^\top \Sigma^{-1}}{1 + \tfrac{1}{\lambda} \hat{k}_t^\top \Sigma^{-1} \hat{k}_t}.
\end{aligned}
\end{equation}
Left-multiplying by $\hat{k}_t^\top$ and introducing the scalar $s := \hat{k}_t^\top \Sigma^{-1} \hat{k}_t$,
\begin{equation}
\begin{aligned}
&\hat{k}_t^\top \bigl( \hat{k}_t \hat{k}_t^\top + \lambda \Sigma \bigr)^{-1} \\
&= \frac{\hat{k}_t^\top \Sigma^{-1}}{\lambda} - \frac{\tfrac{s}{\lambda^2} \hat{k}_t^\top \Sigma^{-1}}{1 + \tfrac{s}{\lambda}} \\
&= \frac{\hat{k}_t^\top \Sigma^{-1}}{\lambda} \left( 1 - \frac{s/\lambda}{1 + s/\lambda} \right) \\
&= \frac{\hat{k}_t^\top \Sigma^{-1}}{\lambda + s}.
\end{aligned}
\end{equation}
Multiplying by $R_t$ on the left and substituting $s = \hat{k}_t^\top \Sigma^{-1} \hat{k}_t$ yields the compact spectral form:
\begin{equation}
\label{eq:spectral_update}
\Delta_t = \frac{R_t \hat{k}_t^\top \Sigma^{-1}}{\lambda + \hat{k}_t^\top \Sigma^{-1} \hat{k}_t}.
\end{equation}

\paragraph{Per-Direction Allocation in the Eigenbasis of $\Sigma$.}
We now project Equation~\ref{eq:closed_form_single} into the eigenspace of $\Sigma$ to derive the per-direction update allocation. Since $\Sigma$ is symmetric positive-definite, it admits the eigendecomposition $\Sigma = \Psi \Lambda \Psi^\top$ with $\Psi = [\psi_1, \dots, \psi_n]$ an orthonormal basis and $\Lambda = \operatorname{diag}(\lambda_1, \dots, \lambda_n)$. Correspondingly,
\begin{equation}
\Sigma^{-1} = \sum_{l=1}^n \frac{1}{\lambda_l} \psi_l \psi_l^\top.
\end{equation}
Defining $\tilde{k}_l := \psi_l^\top \hat{k}_t$ as the projection of the effective key onto the $l$-th principal direction, substituting into the vector factor $\hat{k}_t^\top \Sigma^{-1}$ gives
\begin{equation}
\label{eq:hk_sigma_inv_expand}
\hat{k}_t^\top \Sigma^{-1} = \sum_{l=1}^n \frac{\hat{k}_t^\top \psi_l}{\lambda_l} \psi_l^\top = \sum_{l=1}^n \frac{\tilde{k}_l}{\lambda_l} \psi_l^\top.
\end{equation}
To read off the contribution along the $j$-th principal direction, we right-multiply Equation~\ref{eq:hk_sigma_inv_expand} by $\psi_j$ and apply the orthonormality $\psi_l^\top \psi_j = \delta_{lj}$:
\begin{equation}
\bigl( \hat{k}_t^\top \Sigma^{-1} \bigr) \psi_j = \sum_{l=1}^n \frac{\tilde{k}_l}{\lambda_l} (\psi_l^\top \psi_j) = \frac{\tilde{k}_j}{\lambda_j}.
\end{equation}
Since the scalar prefactor $R_t / (\lambda + s)$ in Equation~\ref{eq:spectral_update} is independent of $j$, the update magnitude along $\psi_j$ is proportional to
\begin{equation}
\label{eq:spectral_allocation}
\Delta W_{\text{dir}\ j} \propto \frac{\tilde{k}_j}{\lambda_j}.
\end{equation}
The numerator $\tilde{k}_j$ quantifies the \emph{semantic demand} of the edit along direction $\psi_j$, while the denominator $\lambda_j$ reflects the \emph{inertia} ( statistical energy) of the pre-trained manifold along that direction. Their ratio formalizes the principle of \emph{impedance-matched editing}: editing energy is channeled into directions that are simultaneously relevant to the target knowledge and structurally under-utilized by the pre-trained model. This establishes Proposition~\ref{prop:soft}: the closed-form update is energy-adaptive in the eigenbasis of $\Sigma$, allocating more capacity to under-utilized low-energy directions and less to dominant high-energy directions.

\paragraph{Remark: the general multi-key case.}
The single-key assumption yields the scalar form above, but it is not essential for the energy-adaptive behavior. For a general key matrix, write $A_t := P_{t-1} K_t \in \mathbb{R}^{n \times b}$, so that the closed-form solution reads $\Delta_t = R_t A_t^\top (A_t A_t^\top + \lambda \Sigma)^{-1}$. Applying the Woodbury identity to $(\lambda \Sigma + A_t A_t^\top)^{-1}$ and using $I - G_t (\lambda I + G_t)^{-1} = \lambda (\lambda I + G_t)^{-1}$ with $G_t := A_t^\top \Sigma^{-1} A_t$ gives
\begin{equation}
\label{eq:closed_form_multikey}
\Delta_t = R_t \bigl( \lambda I + G_t \bigr)^{-1} A_t^\top \Sigma^{-1},
\end{equation}
which reduces to Equation~\ref{eq:spectral_update} when $b = 1$. Since $\Sigma^{-1} \psi_j = \psi_j / \lambda_j$, the per-direction allocation becomes
\begin{equation}
\Delta_t \psi_j = \frac{1}{\lambda_j} \, R_t \bigl( \lambda I + G_t \bigr)^{-1} A_t^\top \psi_j,
\end{equation}
where the matrix factor $R_t (\lambda I + G_t)^{-1}$ does not depend on $j$. The $1 / \lambda_j$ scaling therefore persists in the general case, with the scalar semantic demand $\tilde{k}_j$ replaced by the vector $A_t^\top \psi_j$.

\subsection{Proof of Proposition~\ref{prop:hard}}
\label{app:preservation_proof}

\paragraph{Recursive Definition of $P_t$.}
Setting $\epsilon = 0$ and replacing the inverse with the Moore--Penrose pseudo-inverse, the recursion reads
\begin{equation}
\label{eq:Pt_recursion}
P_t = P_{t-1} - P_{t-1} K_t \bigl(K_t^\top P_{t-1} K_t\bigr)^{+} K_t^\top P_{t-1}, \qquad P_0 = I.
\end{equation}
The pseudo-inverse coincides with the ordinary inverse whenever $K_t^\top P_{t-1} K_t$ is invertible. Proposition~\ref{prop:hard} concerns this idealized recursion; in practice we use $\epsilon = 10^{-6}$ to keep the matrix invertible.

\begin{proof}[Proof of Proposition~\ref{prop:hard}]
We prove by induction on $t \ge 0$ that $P_t$ is symmetric and idempotent and that $P_t K_i = 0$ for all $i \le t$. In the base case, $P_0 = I$ is symmetric and idempotent, and the condition on historical keys is vacuous.

\textbf{Inductive step.} Assume the claim holds for $t-1$ and write $A_t := P_{t-1} K_t$ and $\Pi_t := A_t (A_t^\top A_t)^{+} A_t^\top$. Since $P_{t-1}$ is symmetric and idempotent, $K_t^\top P_{t-1} K_t = A_t^\top A_t$ and $P_{t-1} K_t (K_t^\top P_{t-1} K_t)^{+} K_t^\top P_{t-1} = \Pi_t$, so $P_t = P_{t-1} - \Pi_t$. By the properties of the pseudo-inverse, $\Pi_t$ is the orthogonal projector onto $\mathrm{range}(A_t)$, and hence $\Pi_t A_t = A_t$.

\emph{Projector property.} Since $P_{t-1} A_t = P_{t-1}^2 K_t = A_t$, we have $P_{t-1} \Pi_t = \Pi_t$ and, by symmetry, $\Pi_t P_{t-1} = \Pi_t$. Therefore $P_t^\top = P_t$ and
$P_t^2 = P_{t-1}^2 - P_{t-1} \Pi_t - \Pi_t P_{t-1} + \Pi_t^2 = P_{t-1} - \Pi_t = P_t$.

\emph{Historical keys ($i < t$).} By the inductive hypothesis, $P_{t-1} K_i = 0$ and $A_t^\top K_i = K_t^\top P_{t-1} K_i = 0$. Hence $P_t K_i = P_{t-1} K_i - A_t (A_t^\top A_t)^{+} A_t^\top K_i = 0$.

\emph{Current key ($i = t$).} Since $A_t^\top K_t = K_t^\top P_{t-1} K_t = A_t^\top A_t$, we have $\Pi_t K_t = A_t (A_t^\top A_t)^{+} A_t^\top A_t = \Pi_t A_t = A_t$. Hence $P_t K_t = P_{t-1} K_t - \Pi_t K_t = A_t - A_t = 0$.

This completes the induction and proves (i).

\emph{Rank bound (ii).} Since $\mathrm{rank}(\Pi_t) = \mathrm{rank}(A_t) \le \mathrm{rank}(K_t)$, the subadditivity of rank gives $\mathrm{rank}(P_t) \ge \mathrm{rank}(P_{t-1}) - \mathrm{rank}(K_t)$. Unrolling this inequality from $\mathrm{rank}(P_0) = n$ yields $\mathrm{rank}(P_t) \ge n - \sum_{i=1}^{t} \mathrm{rank}(K_i)$.
\end{proof}

\begin{corollary}
  \label{cor:leakage}
  Let $S_{t-1} := \mathrm{span}\{K_i\}_{i<t}$ and adopt the setting of
  Proposition~\ref{prop:hard} ($\epsilon = 0$). Then the deployed update
  $\Delta_t = R_t F_t$ of Equation~\ref{eq:closed_form_multikey} satisfies
  \begin{equation}
  \mathrm{range}(\Delta_t^\top) \;\subseteq\; \Sigma^{-1}\bigl(S_{t-1}^{\perp}\bigr)
  \;=\; \bigl(\Sigma\, S_{t-1}\bigr)^{\perp},
  \end{equation}
  and its interference with any historical key block $K_i$ ($i<t$) is given exactly by
  \begin{equation}
  \label{eq:leak_identity}
  \Delta_t K_i \;=\; R_t (\lambda I + G_t)^{-1}\, K_t^\top
  \bigl[P_{t-1},\, \Sigma^{-1}\bigr] K_i ,
  \end{equation}
  where $[X,Y] := XY - YX$. Consequently, if $S_{t-1}$ is $\Sigma$-invariant
  (equivalently, spanned by eigenvectors of $\Sigma$) then $\Delta_t K_i = 0$
  exactly.
\end{corollary}
  
\begin{proof}
  Since $G_t = A_t^\top \Sigma^{-1} A_t$ is symmetric, so is $(\lambda I + G_t)^{-1}$,
  and therefore $F_t^\top = \Sigma^{-1} A_t (\lambda I + G_t)^{-1}$. From
  $\Delta_t = R_t F_t$ we get $\Delta_t^\top = F_t^\top R_t^\top$, hence
  $\mathrm{range}(\Delta_t^\top) \subseteq \mathrm{range}(F_t^\top)
  \subseteq \Sigma^{-1}\,\mathrm{range}(A_t)$.
  As $A_t = P_{t-1}K_t$ we have $\mathrm{range}(A_t) \subseteq \mathrm{range}(P_{t-1})$.
  Moreover $P_{t-1}$ is symmetric (as shown in the proof of
  Proposition~\ref{prop:hard}), so Proposition~\ref{prop:hard}(i) gives
  $K_i^\top P_{t-1} = (P_{t-1} K_i)^\top = 0$ for every $i < t$; hence every vector in
  $\mathrm{range}(P_{t-1})$ is orthogonal to $S_{t-1}$, that is,
  $\mathrm{range}(P_{t-1}) \subseteq S_{t-1}^{\perp}$.
  The identification
  $\Sigma^{-1}(S_{t-1}^{\perp}) = (\Sigma S_{t-1})^{\perp}$ follows from
  $\langle \Sigma x, s\rangle = \langle x, \Sigma s\rangle$.
  For Equation~\ref{eq:leak_identity}, substitute
  $P_{t-1}\Sigma^{-1} = \Sigma^{-1}P_{t-1} + [P_{t-1},\Sigma^{-1}]$ into
  $\Delta_t K_i = R_t(\lambda I + G_t)^{-1} K_t^\top P_{t-1}\Sigma^{-1}K_i$ and use
  $P_{t-1}K_i = 0$. If $\Sigma S_{t-1} = S_{t-1}$ then
  $(\Sigma S_{t-1})^{\perp} = S_{t-1}^{\perp}$ and the first claim already forces
  $\Delta_t K_i = 0$.
\end{proof}

Corollary~\ref{cor:leakage} makes precise what recursive null-space alignment does and does not
  guarantee. The recursion of Equation~\ref{eq:recursive_null_space_alignment}
  builds a \emph{Euclidean} orthogonal projector, so it enforces
  $K_i^\top (P_{t-1}K_t) = 0$ for every $i<t$; the interference of the deployed update, however, is
  measured in the $\Sigma^{-1}$ inner product, $K_i^\top \Sigma^{-1}(P_{t-1}K_t)$.
  The two coincide when $S_{t-1}$ is $\Sigma$-invariant, and the residual
  leakage is therefore attributable entirely to the anisotropy of $\Sigma$. Section~\ref{app:leakage}
  measures the resulting leakage and finds it small
  ($\mathrm{leak}_t = 0.305$ at the median, a $29.6\times$ reduction over
  $P_{t-1} = I$), consistent with the historical key subspace being close to, but
  not exactly, $\Sigma$-invariant.

\section{Extended Related Work}
\label{app:related_work}

Knowledge editing aims to precisely modify specific knowledge encoded in Large Language Models (LLMs) while preserving unrelated capabilities. The existing literature is primarily divided into memory-based methods and model editing methods.

\subsection{Memory-based Knowledge Editing}
Memory-based approaches maintain an external repository of edited facts to avoid direct parameter interference. These can be further categorized by the nature of the stored memory:
\paragraph{Explicit Memory.} Early methods like IKE~\cite{zheng2023edit} leverage In-Context Learning (ICL) by appending demonstrations to the prompt. SERAC~\cite{mitchell2022memory} utilizes a memory bank and a classifier to route queries to an additional model when an edit is triggered. 
\paragraph{Implicit Memory.} GRACE~\cite{hartvigsen2023aging} maintains a codebook at a specific model layer, using previous layer outputs as keys and learnable vectors (trained on target answers) as values to replace the original layer output.

\paragraph{Parametric Memory.} This category introduces additional parameter modules without altering the base weights. MELO~\cite{yu2023melo} trains chunk-wise LoRA matrices based on edit clusters, while Transformer-Patcher~\cite{huang2023transformer} adds new neurons to the final FFN layer to fix individual errors. 
A range of subsequent works further explores augmenting language models with lightweight calibration, intervention, or memory modules attached to specific layers to incorporate editable knowledge~\cite{dong2022calibrating,wang2026memoir,zhanglanguage,liu2026representation}.

\paragraph{Limitations of Memory-based Methods.}
Relying heavily on retrieval accuracy, these methods struggle with generalized or paraphrased queries targeting edited facts. Furthermore, their dependence on external memory introduces fundamental bottlenecks in scalability, incurring continuously growing storage requirements and escalating inference-time computational overhead. In addition, this reliance tightly couples the model to its external memory infrastructure, severely constraining deployment flexibility and making cross-platform portability impractical.

\subsection{Model Editing Methods}
Model editing seeks to permanently ``bake'' knowledge into the model's parameters. This can be achieved through hyper-networks, localized optimization, or global training.

\paragraph{Hypernetwork-based Paradigm.} This family of methods trains a hyper-network to produce targeted parameter modifications that encode edited knowledge. For example, MEND~\cite{mitchell2022fast} and KnowledgeEditor~\cite{de-cao2021editing} learn to predict weight deltas from gradients derived from new factual triples. However, these approaches introduce additional parameters and training overhead. Moreover, since hyper-networks remain frozen after pretraining, they struggle to accommodate
continual edits, which significantly limits their practical applicability. 

\paragraph{Locate-then-Edit Paradigm.} This paradigm identifies critical ``knowledge neurons'' before applying targeted updates. KN~\cite{dai2022knowledge} identifies knowledge neurons via attribution and updates specific FFN values. ROME~\cite{meng2023locating} uses causal tracing to locate a single layer and updates its weights using a closed-form solution. MEMIT~\cite{meng2023mass} extends ROME to handle large-scale edits by distributing the update across multiple layers. Further extensions include MEMITCSK~\cite{gupta2023editing} (considering predicates), PMET~\cite{li2024pmet} (including MHSA layers), and DEM~\cite{huang2024commonsense} (leveraging hidden state similarity for localization). UltraEdit~\cite{gu2026ultraedit} performs editing without pre-computation of the second-moment matrix for the original knowledge representations.

To mitigate sequential decay and maintain model stability during continual editing, recent works have explored advanced constraints and weight matrix properties. AlphaEdit~\cite{fang2025alphaedit} introduces a null-space projection to ensure that updates do not interfere with pre-trained features, while EvoEdit~\cite{lyu2025evoedit} employs a dynamic projection matrix to protect historically edited knowledge. 
Unlike these approaches, our recursive null-space alignment does not rely on pre-trained initialization, thereby avoiding the collapse of the feasible-update space and enabling more effective updates.

Some works focus on the properties of the weight matrices, employing regularization to constrain the parameter updates. NAS~\cite{liu2026ultralonghorizon} and MPES~\cite{gupta2025lifelong_regularization} identify that catastrophic forgetting in sequential editing is closely tied to the unconstrained exponential growth of the weight matrix norm; NAS addresses this by constraining the target value vector's norm, whereas MPES introduces a regularization term to penalize deviations from the original weights. Furthermore, REVIVE~\cite{zhang2026spectral_sequential} observes that general semantic knowledge is stored in the principal singular space of the parameters, protecting it by filtering out updates in this space. PRUNE~\cite{ma2025perturbation_restrained} notes that sequential editing increases the condition number of weight matrices, thereby degrading model capacity; it mitigates this by suppressing the singular values of the cumulative edit updates rather than individual ones. Finally, MOSE~\cite{xu2026multiplicative_orthogonal} replaces the standard additive update with a multiplicative orthogonal update matrix to better preserve matrix norms and condition numbers. In contrast to these approaches, which impose rigid constraints on weight norms or condition numbers, our energy-weighted penalty preserves manifold structure, naturally and effectively safeguarding general model capabilities.

Traditional methods are limited to atomic triplets. Recent works have expanded editing to unstructured long-form text. UnKE~\cite{deng2025everything} treats the model as a key-encoder and value-generator, performing block-level localization and gradient-descent-based optimization.
FT-UKE~\cite{xiong2025finetuning} provides a training recipe that makes vanilla fine-tuning effective for unstructured knowledge editing. EvoEdit-Cao~\cite{cao2025evoedit} addresses lifelong editing through latent perturbation and a knowledge-driven parameter fusion mechanism. 
To adapt the locate-then-edit paradigm to unstructured long text, AnyEdit~\cite{jiang2025anyedit} decomposes long-form text into uniform chunks and iteratively edits the last token of each chunk. $\mu$KE~\cite{su2025muke} further introduces a Matryoshka-style loss to capture multi-chunk dependencies, while AnyEdit++~\cite{tian2026anyeditpp} instead places chunk boundaries at peaks of Bayesian Surprise. Instead of heuristically segmenting input samples, which yields the mediocre-point dilemma, our Pivot Localization identifies easily displaceable regions on the edit sub-manifold, therefore enabling more effective unstructured knowledge updates.

\vspace{-5pt}
\section{Limitations}

Despite ManiEdit's strong performance in sequential unstructured knowledge editing, it still faces a key limitation: ManiEdit is currently confined to textual knowledge editing and lacks support for multimodal integration. Extending our manifold-based paradigm to accommodate cross-modal updates (e.g., text, images, and audio) would significantly broaden its applicability, presenting a promising direction for future research in multimodal large language models.

\end{document}